\documentclass[11pt]{article}
\usepackage{amsmath,amsfonts,amssymb,amsthm}
\usepackage{xcolor}
\usepackage{comment}
\usepackage{bm}
\usepackage{graphicx}
\usepackage{geometry}
\usepackage{paralist}
\usepackage{booktabs}
\usepackage{authblk}
\usepackage{url}
\usepackage{natbib}

\renewcommand{\emptyset}{\varnothing}
\renewcommand{\le}{\leqslant}
\renewcommand{\ge}{\geqslant}

\newcommand{\ind}{\boldsymbol{1}}

\newcommand{\bsx}{\boldsymbol{x}}
\newcommand{\bse}{\boldsymbol{e}}
\newcommand{\bsy}{\boldsymbol{y}}
\newcommand{\bsz}{\boldsymbol{z}}
\newcommand{\real}{\mathbb{R}}
\newcommand{\cx}{\mathcal{X}}
\newcommand{\bszero}{\boldsymbol{0}}
\newcommand{\bsone}{\boldsymbol{1}}

\newcommand{\tran}{\mathsf{T}}

\newcommand{\e}{\mathbb{E}}

\newcommand{\rd}{\,\mathrm{d}}

\newcommand{\phm}{\phantom{-}}

\newcommand{\auc}{\mathrm{AUC}}
\newcommand{\abc}{\mathrm{ABC}}
\newcommand{\aul}{\mathrm{AUL}}

\newcommand{\giv}{\!\mid\!}

\newcommand{\phe}{\phantom{=}}

\newtheorem{proposition}{Proposition}
\newtheorem{lemma}{Lemma}
\newtheorem{theorem}{Theorem}

\begin{document}

\title{Deletion and Insertion Tests in Regression Models}
\date{}
\author[1]{Naofumi Hama}
\author[1]{Masayoshi Mase}
\author[2]{Art B. Owen}
\affil[1]{Hitachi, Ltd.\\Research \& Development Group}
\affil[2]{Department of Statistics\\Stanford University}
\affil[ ]{\textit {naofumi.hama.hd@hitachi.com, masayoshi.mase.mh@hitachi.com, owen@stanford.edu}}

\maketitle

\begin{abstract}
    A basic task in explainable AI (XAI) is to identify the most important features behind a prediction made by a black box function $f$.  The insertion and deletion tests of \cite{petsiuk2018rise} can be used to judge the quality of algorithms that rank pixels from most to least important for a classification.  Motivated by regression problems we establish a formula for their area under the curve (AUC) criteria in terms of certain main effects and interactions in an anchored decomposition of $f$.  We find an expression for the expected value of the AUC under a random ordering of inputs to $f$ and propose an alternative area above a straight line for the regression setting.  We use this criterion to compare feature importances computed by integrated gradients (IG) to those computed by Kernel SHAP (KS) as well as LIME, DeepLIFT, vanilla gradient and input$\times$gradient methods. KS has the best overall performance in two datasets we consider but it is very expensive to compute.
    We find that IG is nearly as good as KS while being much faster.
    Our comparison problems include some binary inputs that pose a challenge to IG because it must use values between the possible variable levels and so we consider ways to handle binary variables in IG.
    We show that sorting variables by their Shapley value does not necessarily give the optimal ordering for an insertion-deletion test. It will however do that for monotone functions of additive models, such as logistic regression.
\end{abstract}

\section{Introduction}

Explainable AI methods are used help humans
learn from patterns that a machine learning or artificial intelligence model has found, or to
judge whether those patterns are scientifically
reasonable or whether they treat subjects fairly.
As \cite{hooker2018benchmark} note, there is no
ground truth for explanations. \cite{mase2022variable} attribute this to the greater difficulty of identifying causes of effects
compared to effects of causes \citep{dawi:musi:2021}.
Lacking a ground truth, researchers turn to axioms
and sanity checks to motivate and vet explanatory methods.  There are also some numerical measures that one can use to compute a quality measure for methods that rank variables from most to least important.  These include the Area Over Perturbation Curve (AOPC) of \cite{samek2016evaluating} and the Area Under the Curve (AUC) measure of \cite{petsiuk2018rise} that we focus on.
They have the potential to augment intuitive and philosophical distinctions among methods with precise numerical comparisons. In this
paper we make a careful study of the properties
of those measures and we illustrate their use on two datasets.

Insertion and deletion tests were used by \cite{petsiuk2018rise} to compare variable importance methods for black box functions.  In their specific case they had an image classifier that would, for example, conclude with high confidence that a given image contains a mountain bike. Then the question of interest was to identify which pixels are most important to that decision.  They propose to delete pixels, replacing them by a plain default value
(constant values such as black or the average of all pixels from many images)
in order from most to least important for the decision that the image was of the given class.  If they have ordered the pixels well, then the confidence level of the predicted class should drop quickly as more pixels are deleted. By that measure, their Randomized Input Sampling for Explanation (RISE) performed well compared to alternatives such as GradCAM \citep{selvaraju2017grad} and LIME \citep{ribeiro2016should} when explaining outputs of a ResNet50 classifier \citep{zhang2018top}.  For instance, Figure 2 of \cite{petsiuk2018rise} has an example where occlusion of about 4\% of pixels, as sorted by their RISE criterion can overturn the classification of an image.
They also considered an insertion test starting from
a blurred image of the original one
and inserting pixels from the real image in order from most to least important. An ordering where the confidence rises most quickly is then to be preferred.
\cite{petsiuk2018rise} scored their methods by an area under the curve (AUC) metric that we will describe in detail below. The idea to change features in order of importance and score how quickly predictions change is quite natural, and we expect many others have used it. We believe that our analysis of the methods in \cite{petsiuk2018rise} will shed light on other similar proposals.

Figure~\ref{fig:cub} shows an example where an image is correctly and confidently classified as an albatross by an algorithm described in Appendix~\ref{sec:imageexample}.  The integrated gradients (IG) method of \cite{sundararajan2017axiomatic} that we define below can be used to rank the pixels by importance.  In this instance the deletion AUC is $0.27$ which can be interpreted as meaning that about 27\% of the pixels
have to be deleted before the algorithm completely forgets that the image is of an albatross.
The model we used accepts square images of size $224\times224$ pixels with 3 color channels. As a preprocessing step, we cropped the leftmost image in Figure~\ref{fig:cub} to its central $224\times224$ pixels.
The IG feature attributions from each pixel (summed over red, green and blue channels) of this square image are presented in the rightmost panel of Figure~\ref{fig:cub}.
A saliency map shows that pixels in the bird's face, especially the eye and beak are rated as most important.

    {In this paper we study insertion and deletion metrics for uses that include regression problems in addition to classification.} We consider a function $f(\bsx)$ such as an estimate of a response $y$ given $n$ predictors represented as components of $\bsx$.  The regression context is different from classification. The trajectory taken by $f(\cdot)$ as inputs are switched one by one from a point $\bsx$ to a baseline point $\bsx'$ can be much less monotone than in the image classification problems that motivated \cite{petsiuk2018rise}.  We don't generally find that either $f(\bsx)$ or $f(\bsx')$ is near zero.  There are also use cases where $\bsx$ and $\bsx'$ are both actual observations; it is not necessary for one of them to be an analogue of a completely gray image or otherwise null data value. Sometimes $f(\bsx)\approx f(\bsx')$ and yet it can still be interesting to understand what happens to $f$ when some components of $\bsx$ are changed to corresponding values of $\bsx'$.

Our main contributions are as follows.
Despite these differences between classification and regression, we find that insertion and deletion metrics can be naturally extended to regression problems. We then develop expressions for the resulting AUC in terms of certain main effects and interactions in $f$ building on the anchored decomposition from \cite{kuo:sloa:wasi:wozn:2010} and others. This anchored decomposition is a counterpart to the better known analysis of variance (ANOVA) decomposition. The anchored decomposition does not require a distribution on its inputs, much less the independence of those inputs that the ANOVA requires.  In the regression context we prefer to change the AUC computation replacing the horizontal axis by a straight line connecting $f(\bsx)$ to $f(\bsx')$.  We obtain an expression for the average AUC in a case where variables were inserted in a uniform random order over all possible permutations.  In settings without interactions the area between the variable change curve (that we define below) and the straight line has expected value zero under those permutations, but interactions change this. We also show that the expected area between the insertion curve and an analogous deletion curve that we define below does have expected value zero, even in the presence of interactions of any order. Some other contributions described below show that in some widely used models the same ordering that optimizes an area criterion also optimizes a Shapley value.

We take a special interest in the integrated gradients (IG) method of \cite{sundararajan2017axiomatic} because it is very fast.
The number of function or derivative evaluations that it requires grows only linearly in the number of input variables, for any fixed number of evaluation nodes in the Riemann sum it uses to approximate an integral.
The cost of exact computation for kernel SHAP (KS) of \cite{lundberg2017unified} grows exponentially with the number of variables, although it can be approximated by sampling.
We also include LIME of \cite{ribeiro2016should},
DeepLIFT of \cite{shrikumar2017learning},
Vanilla Grad of \cite{simonyan2013deep} and input times gradient method of \cite{shrikumar2016not}.
In the datasets we considered, KS is generally
best overall. We note that the term `Vanilla' is not used by \cite{simonyan2013deep} to describe their methods but it has been used by others, such as \cite{agarwal2022openxai}.

It is very common for machine learning
functions to include binary inputs.
For this reason we discuss how to extend IG to handle some dichotomous variables and then compare it to
the other methods, especially KS. Simply extending the domain of $f$ for such variables from $\{0,1\}$ to $[0,1]$ is easy to do and it avoids the exponential cost that some more principled
choices have.

The remainder of this paper is organized as following.
Section~\ref{sec:back} cites related works, introduces some notation and places our paper in the context of explainable AI (XAI),
{while also citing some works
        that express misgivings about XAI.}
Section~\ref{sec:aucreg} defines the AUC and
gives an expression for it in terms of main effects
and interactions derived from an anchored decomposition
of the prediction function $f$. The expected AUC is obtained for a random ordering of variables.
We also introduce an area between the curves (ABC) quantity using a linear interpolation baseline curve instead of the horizontal axis.
    {We show that arranging input variables
        in decreasing order by Shapley value does not necessarily
        give the order that maximizes AUC, due to the
        presence of interactions. Models that, like logistic regression
        are represented by an increasing function of an additive
        model do get their greatest AUC from the Shapley ordering
        depsite the interactions introduced by the increasing function. When that function is differentiable, then IG finds the optimal order.
    }
Section~\ref{sec:ig4binary}
discusses how to extend IG to some dichotomous variables.
We consider three schemes: simply treating binary variables in $\{0,1\}$ as if they were continuous values in $[0,1]$,
multilinear interpolation of the function values
at binary points, and using paths that jump from $x_j=0$ to $x_j=1$.
The simple strategy of casting the
binary inputs to $[0,1]$, which many prediction functions can do, is preferable on grounds of speed.
Section~\ref{sec:exp} presents some empirical work.  We choose a regression problem about explaining the value of a house in Bangalore using some data from Kaggle (Section~\ref{sec:expindia}).  This is a challenging prediction problem because the values are heavily skewed.  It is especially challenging for IG because all but two of the input variables are binary and IG is defined for continuous variables. The model we use is a multilayer perceptron.  We compare variable rankings from
six methods.
KS is overall best but IG is much faster
and nearly as good.
Section~\ref{sec:cernexp} looks at a problem from high energy physics using data from CERN.
Section~\ref{sec:roarexample} includes
a ROAR analysis that compares how well the KS and IG measures rank the importance of variables in a
setting where the model is to be retrained without
its most important variables.
Section~\ref{sec:con} has some final comments.
Appendix~\ref{sec:modeldetails} gives some details of the data and models we study. Theorem~\ref{thm:aucfromdeltaf} is proved in Appendix~\ref{sec:auctheory}.

\section{Related Work}\label{sec:back}

The insertion and deletion measures we study are part of XAI.
Methods from machine learning and artificial intelligence are being deployed in electronic commerce, finance and other industries. Some of those applications are mission critical (e.g., in medicine, security or autonomous driving).  When models are selected based on their accuracy on holdout sets, the winning algorithms can be very complicated. There is then a strong need for human understanding of how the methods work in order to reason about their accuracy on future data.
For discussions on the motivations and methods for XAI
see recent surveys such as \cite{liao2021human}, \cite{saeed2021explainable} and \cite{bodria2021benchmarking}.

One of the most prominent XAI tasks is attribution in which one quantifies and compares importance of the model inputs to the resulting prediction.
Since there is no ground-truth for explanations, these attributions are usually compared based on theoretical justifications.
From this viewpoint,
SHAP (SHapley Additive exPlanations) \citep{lundberg2017unified} and
IG \citep{sundararajan2017axiomatic}
are popular feature importance methods
due to their grounding in cooperative game theory.
Given some reasonable axioms, game theory can produce
unique variable importance measures
in terms of Shapley values and Aumann-Shapley values respectively.

As a complementary method to theoretical a priori justification, one can also examine the outputs of attribution methods numerically, either applying sanity checks or computing quantitative quality measures.  The insertion and deletion tests of \cite{petsiuk2018rise} that we study are of this type.
Quantitative metrics for XAI were recently surveyed  in \cite{nauta2022anecdotal} who also discuss metrics based on whether the importance ratings align with those of human judges.

As an example of sanity checks, \citet{adebayo2018sanity} find that some algorithms to produce saliency maps are nearly unchanged when one scrambles the category labels on which the algorithm was trained, or when one replaces trained weights in a neural network by random weights.  They also find that some saliency maps are very close to what edge detectors would produce and therefore do not make much use of the predicted or actual class of an image.
Another example from \citet{adebayo2020debugging} checks whether saliency mathods can detect spurious correlation in a setting where the backgrounds in training images are artificially correlated to output labels and the model learns this spurious correlation correctly.

One important method is the ROAR (RemOve And Retrain) approach of  \cite{hooker2018benchmark}.
It sequentially removes information from the least important columns in each observation in the training data
and retrains the model with these partially removed training data iteratively.
A good variable ranking will show rapidly decreasing performance of the classifier as the most important inputs are removed.
The obvious downside of this method is that it requires a lot of expensive retraining that insertion and deletion methods avoid.

\begin{figure}
    \includegraphics[width=\linewidth]{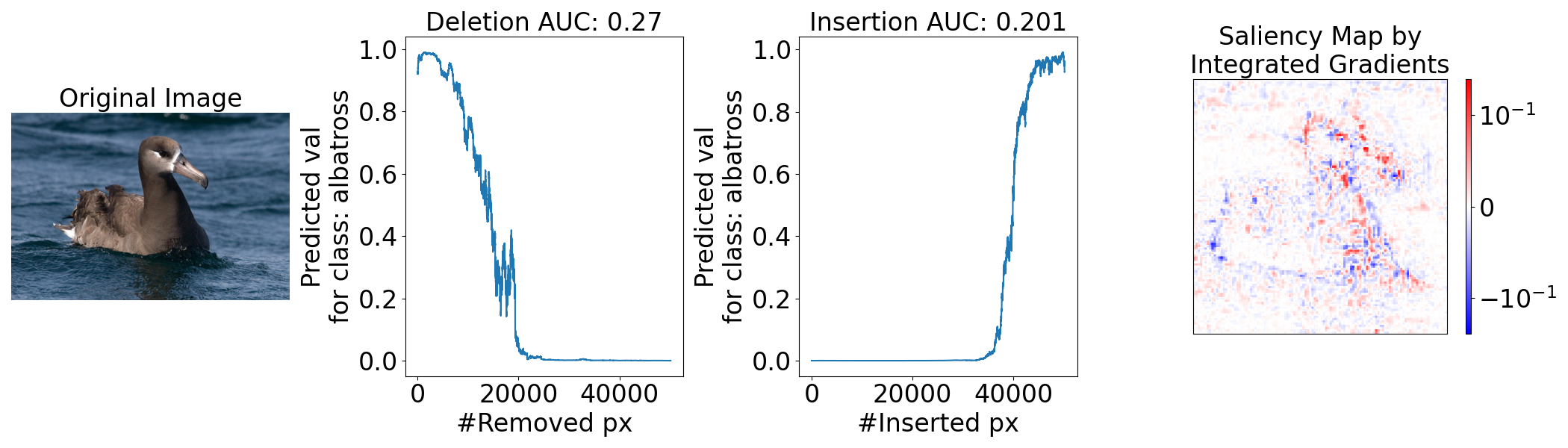}
    \caption{An example of deletion and insertion tests in image classification for an image from \citet{wah2011caltech}.}
    \label{fig:cub}
\end{figure}

The insertion and deletion tests we study have been
criticized by \citet{gomez2022metrics} who note that
the synthesized images generated in these tests are unnatural
and do not resemble the images on which the algorithms
were trained.  The same issue of unnatural inputs was raised by
\cite{mase:owen:seil:2019}.
\cite{gomez2022metrics} also point out that
insertion and deletion tests only compare the rankings of the inputs.

\citet{fel2021look} noted that scores on insertion tests can be strongly
influenced by the first few pixels inserted
into a background image.  They then note that images with only a few non-background pixels are quite different from the target image.
    {
        The policy choice of what baseline to
        compare an image to affects whether the
        result can be manipulated.
        An adversarily selected image might
        differ from a target image in only
        a few pixels, and yet have a very
        different classification.  At the same
        time both of these images differ greatly
        from a neutral background image.  An
        insertion test comparing the target image
        to a blurred background might show that
        the classification depends on many pixels,
        while a deletion test with the adversarial image
        will show that only a few pixels need to change.
        The two tests will thus disagree on whether
        the classification was influenced by few or
        many pixels.  Both are correct because they
        address different issues.
    }

A crucial choice in using insertion and deletion tests
is which reference input to use when deleting or inserting variables.  \cite{petsiuk2018rise} decided to insert real pixels into a blurred image instead of inserting them into a solid gray image because, when the most important pixels form an ellipsoidal shape, then inserting them into a gray background might lead a spurious classification (such as `balloon'). On the other hand, deleting pixels via blurring might underestimate the salience of those pixels if the algorithm is good at inferring from blurred images. In our albatross example of Figure~\ref{fig:cub} we used an all black image with $224\times224\times3$ zeros. We call such choices `policies' and note that a good policy choice depends on the scientific goals of the study and the strengths and weaknesses of the algorithm under study.

\cite{haug2021baselines} study different kinds of baseline images for image classification problems noting that the choice of background affects how well a method performs.  Among the backgrounds they mention are constants, blurred images, uniform or Gaussian noise, average images and baseline images maximally distant from the target image. They also mention the neutral backgrounds of \cite{izzo2020baseline} that lie on a decision boundary. \cite{sundararajan2020many} discuss taking every available data point as a baseline and averaging the resulting Shapley values. \cite{sturmfels2020visualizing} note the importance of choosing baselines carefully, pointing out that zero would be a bad baseline for blood sugar and that if a solid color image is used as a baseline in image classification, then the result cannot attribute importance to pixels of that color.  They discuss many of the baselines that \cite{haug2021baselines} do and they propose the `farthest image' baseline.

There are also broader criticisms of XAI methods.  \cite{kuma:2020} point out some difficulties in formulating a variable importance problem as a game suitable for use in Shapley value. Their view is that those explanations do not match what people expect from an explanation. They also identify a catch-22 where either allowing or disallowing a variable not used by a model to be important causes difficulties. \cite{rudin2019stop} says that one should not use a black box model to make high stakes decisions but should instead only use interpretable models. She also disputes that this would necessarily cause a loss in accuracy. We see a lot of value in that view, and we know that there are examples where interpretable models perform essentially as well as the best black boxes.  However black boxes are very widely used. There is then value in XAI methods that can reveal and quantify any of their flaws.  Furthermore, an explanation depends on not just the form of the model but also on the joint distribution of the predictor variables. For example an interpretable model that does not use the race of a subject might still have a discriminatory impact due to associations among the predictors and XAI methods can be used to evaluate such bias.

Prior uses of insertion and deletion tests have mostly been about image or text classification. There have been a few papers using them for tabular data or time series data, such as
\cite{cai2021xproax}, \cite{hsieh2020evaluations}, \cite{parvatharaju2021learning},
and \cite{ismail2020benchmarking}.

\cite{ancona2017towards} also describe a strategy of changing variables one at a time and observing how the quality of a prediction changes in response independently of \cite{petsiuk2018rise}. Their Figure 3c compares the trajectories taken by several different methods on some image classification problems. They have insertion and deletion curves to compare an occlusion method to integrated gradients.  Unlike \cite{petsiuk2018rise}, they do not report an AUC quantity.

The work of \cite{samek2016evaluating} precedes \cite{petsiuk2018rise} and uses the same global ablation strategy of deleting information in order from most to least important. Their deletions involve replacing a whole block of pixels (e.g., a $9\times 9$ block) by uniformly distributed noise. One motivation for using such noise was to generate images outside the manifold of natural images.  Their average over a perturbation curve is comparable to the average under the deletion curve of \cite{petsiuk2018rise}.  Where \cite{petsiuk2018rise} focus on curves for individual images, \cite{samek2016evaluating} study the average of such curves over many images.  In their numerical work they only make the first 100 such perturbations, affecting about $1/6$ of the pixels.

For us the advantage of the approach in \cite{petsiuk2018rise} is that their AUC is defined by running the variable substitution to completion, changing every feature $x_j$ to the baseline value $x'_j$.  Then we use two orders, one that seeks to increase $f$ as fast as possible and one that seeks to decrease it as fast as possible, and study the area between those two curves.

\section{The AUC for Regression}\label{sec:aucreg}

The AUC method for comparing variable rankings has not had
much theoretical analysis yet.  This section develops some
of its properties.  The development is quite
technical in places.  We begin with a non technical account
emphasizing intuition.
Readers may use the intuitive development as orientation
to the technical parts or they may prefer to skip the
technical parts.

If we order the inputs to a function $f$ and then
change them one at a time from the value in point $\bsx$
to that in a baseline value $\bsx'$,
the resulting function
values trace out a curve over the interval $[0,n]$.
If we have tried to order the variables starting
with those that will most increase $f$ and ending
with those that least increase (most decrease) $f$
then a better ordering is one that gives a larger
area under the curve (AUC).
We will find it useful to consider the
signed area under this curve but above a straight line
connecting the end points. This area between the curves
(ABC) is the original AUC minus the area under the line segment.

A deletion measure orders the variables from those
thought to be most decreasing of $f$ to those that
are least decreasing, i.e., the opposite order to an
insertion test.  For deletion we like to use the
area ABC below the straight line but above the
curve that the deletion process traces out.
When we need to refer to insertion and deletion
ABCs in the same expression we use ABC$^\prime$
for the deletion case.

Additive functions are of special interest
because additivity simplifies explanation
and such models often come close to
the full predictive power of a more
complicated model.
For an additive model, the incremental change in $f$ from
replacing $x_j$ by $x'_j$, call it $\Delta_j$,
does not depend on $x_k$ for $k\ne j$.
In this case the AUC is maximized by ordering
the variables so that $\Delta_1\ge \Delta_2\ge\cdots\ge\Delta_n$.
In this case the Shapley values are $\phi_j=\Delta_j$
and so ordering by Shapley value maximizes both
AUC and ABC.
We also show that if one orders the predictors randomly,
then the expected value of ABC under this randomization
is zero for an additive function.

Prediction functions must also capture interactions among the
input variables.  We study those using some
higher order differences of differences. This way of quantifying
interactions comes from an anchored decomposition that
we present.  This anchored decomposition is a less well known
alternative to the analysis of variance (ANOVA) decomposition.
The interaction quantity for a set $u\subseteq\{1,2,\dots,n\}$
of variables is denoted $\Delta_u$.  This interaction does not
contribute to any points along the curve until the `last' member
of the set $u$, denoted $\lceil u\rceil$, has been changed.
It is thus present only in the final $n+1-\lceil u\rceil$ points
of the insertion curve.  A large AUC comes not just from
bringing the large main effects to the front of the list.  It also
helps to have all elements in a positive interaction $\Delta_u$
included early, and at least one element in a negative interaction
$\Delta_u$ appear very late in the ordering.

When there are interactions present, it is no longer true that
$\e(ABC)$ must be zero under random ordering.
We show in Appendix \ref{sec:abc4del} that the area
between the insertion and deletion curves $\abc+\abc'$
satisfies $\e(\abc+\abc')=0$ under random ordering of inputs
whether or not interactions are present.

Section~\ref{sec:aucvsshapley} shows that if we sort variables in
decreasing order of their Shapley values, then we do not
necessarily get the ordering that maximizes the AUC.
This is natural: the Shapley value $\phi_j$ of variable
$j$ is defined as a weighted sum of $2^{n-1}$ incremental
values for changing $x_j$, while the AUC, uses only one
of those incremental values for variable $j$.  The anchored
decomposition that we present below makes it simple to
construct an example with $n=3$ where the Shapley values
are $\phi_1>\phi_2>\phi_3$ while the order
$(1,3,2)$ has greater AUC than the order $(1,2,3)$.
This is not to say that insertion AUCs are somehow
in error for not being optimized by the Shapley ordering,
nor that Shapley value is in error for not optimizing
the AUC.  The two measures have different definitions
and interpretations.  They can reasonably be considered
proxies for each other, but the Shapley value weights
a variable's interactions in a different way than
the AUC does.

In Appendix~\ref{sec:monotonicity} we consider the logistic
regression model
$f(\bsx)=\Pr( Y=1\giv\bsx)=(1+\exp(-\beta_0-\bsx^\tran\beta))^{-1}$.
Because of the curvature of the logistic transformation,
this function has interactions of all orders.
At the same time $\tilde f(\bsx)=\log(f(\bsx)/(1-f(\bsx)))=
    \beta_0+\bsx^\tran\beta$
is additive so on this scale the Shapley ordering
does maximize AUC.  The AUC on the original probability
scale is perhaps the more interpretable choice.
We show that due to the monotonicity of
the logistic transformation, the Shapley ordering
for $f(\bsx)$ is the same as for $\tilde f(\bsx)$
and so it also maximizes the $\auc$ for $f(\bsx)=\Pr(Y=1\giv\bsx)$.
Because $\exp(\cdot)$ is strictly monotone the Shapley
ordering also optimizes AUC for loglinear models and
for naive Bayes. It is also shown there that
for a differentiable increasing function of an
additive function that integrated gradients will
compute the optimal order.

The next subsections present the above findings
in more detail.  Some of the derivations are
in an appendix.
We use well known properties of the Shapley value.
Those are discussed in many places.  We can
recommend
the recent reference by
\cite{plischke2021computing} because it also
discusses Harsanyi dividends and is motivated
by variable importance.

\subsection{ABC Notation}
We study an algorithm $f:\cx\to\real$ that makes a
prediction based on input data $\bsx\in\cx = \prod_{j=1}^n\cx_j$.
The points $\bsx\in\cx$ are written $\bsx=(x_1,x_2,\dots,x_n)$. In most applications $\cx_j\subseteq\real$. While some attribution methods require real-valued features, the AUC quantity we present does not require it. For classification, $f$ could be the estimated probability that a data point with features $\bsx$ belongs to class $y$, or it could be that same probability prior to a softmax normalization.  Our emphasis is on regression problems.

The set of variable indices is $1{:}n\equiv\{1,2,\dots,n\}$.
For any $u\subseteq1{:}n$ we write $\bsx_u$ for
the components $x_j$ that have $j\in u$.
We write $-u$ for $1{:}n\setminus u$.
We often need to merge indices from two
or more points into one hybrid point.  For this
$\bsx_u{:}\bsx'_{-u}$ is the point $\tilde\bsx\in\cx$
with $\tilde x_j=x_j$ for $j\in u$ and $\tilde x_j=x'_j$
for $j\not\in u$.  That is, the parts of $\bsx$ and $\bsx'$
have been properly assembled in such a way that we can pass
the hybrid to $f$ getting $f(\tilde \bsx)$.
More generally for disjoint $u,v,w$ with $u\cup v\cup w=1{:}n$
the point $\bsx_u{:}\bsy_v{:}\bsz_w$ has components $x_j$, $y_j$ and $z_j$ for $j$ in $u$, $v$ and $w$ respectively.

The cardinality of $u$ is denoted $|u|$.
We also write $\lceil u\rceil = \max\{j\in1{:}n\mid j\in u\}$
with $\lceil \emptyset\rceil =0$ by convention.
It is typographically convenient to shorten the
singleton $\{j\}$ to just $j$ where it could only
represent a set and not an integer, especially within subscripts.

Suppose that we have two points $\bsx,\bsx'\in \cx$
and are given a method to attribute the difference $f(\bsx')-f(\bsx)$ to the variables $j\in1{:}n$.
This method produces attribution values $A_f(\bsx,\bsx')\in\real^n$ with the interpretation that
$A_f(\bsx,\bsx')_j$ is a measure of the effect on
$f$ of changing $x_j$ to $x'_j$.
We can then sort the variables
$j\in1{:}n$ according to their attribution values $A_f(\bsx,\bsx')_j$.
In an insertion test we insert the variables from $\bsx'$
into $\bsx$ in order from ones thought to most increase
$f(\cdot)$ (i.e., largest $A_f(\bsx,\bsx')_j$) to ones thought to most decrease $f(\cdot)$ (smallest $A_f(\bsx,\bsx')_j$).
Let the constructed points be $\tilde\bsx^{(j)}$ for $j=0,1,\dots,n$
with $\tilde\bsx^{(0)}=\bsx$ and  $\tilde\bsx^{(n)}=\bsx'$.
If we have chosen a good order there will be a large area
under the curve $(j,f(\tilde\bsx^{(j)}))$ for $j=0,1,\dots,n$
and consequently also a large (signed) area between
that curve and a straight line connecting its endpoints.
The left panel in Figure \ref{fig:sign_auc} illustrates  ABC for insertion.

In a deletion measure we order the variables from the ones thought
to have the most negative effect on $f$ to the ones thought
to have the most positive effect.  Those variables are changed from $x_j$ to $x_j'$ in that
order and a good ordering creates a curve with a small area
under it.  We keep score by using the signed area above that
curve but below the straight line connecting $f(\bsx)$ to $f(\bsx')$.
Note that we are still inserting variables from
$\bsx'$ into $\bsx$ but, in an analogy to what happens in images
we are deleting the information that we think would make $f$ large,
which in that setting made the algorithm confident about what
was in the image.
Let $\check\bsx^{(j)}$ be the point we get after placing
the $j$ elements of $\bsx'$ thought to most decrease $f$
into $\bsx$.
Our ABC criterion for deletion is the signed area above
the curve $(j,f(\check\bsx^{(j)}))$ but below the straight
line connecting connecting the endpoints.
The right panel in Figure \ref{fig:sign_auc} illustrates  ABC for deletion.

We also considered taking insertion to mean replacing components $x_j$ by $x'_j$ in increasing order of predicted change to $f$ when $f(\bsx')>f(\bsx)$ and taking deletion to mean replacement
starting with the most negative changes when $f(\bsx')<f(\bsx)$.  This convention may seem like a natural
extension of the uses in image classification, but it has two difficulties for regression. First, it is not well defined when $f(\bsx)=f(\bsx')$. Second, while this exact equality might seldom hold, that definition makes cases $f(\bsx')=f(\bsx)+\epsilon$ very different from those with $f(\bsx') = f(\bsx)-\epsilon$, for small $\epsilon>0$.

\begin{figure}[t]
    \centering
    \includegraphics[width=\linewidth]{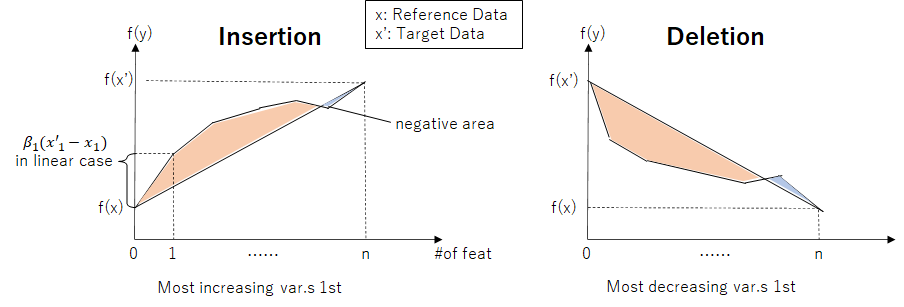}
    \caption{
        The left panel shows the (signed) area between
        a straight line and
        the curve formed by changing components of $\bsx'$
        to those of $\bsx$ after ordering by their
        estimated positive impact on $f$. The right
        panel shows a signed area for deletion of
        variables.
    }
    \label{fig:sign_auc}
\end{figure}

When $\bsx$ and $\bsx'$ are two randomly chosen
data points there is a natural symmetry between
insertion and deletion.  In many settings however,
one of the points $\bsx$ or $\bsx'$ is not an
actual observation but is instead a reference value
such as the gray images discussed above.
As mentioned above, choices of $\bsx$ and $\bsx'$
to pair with each other are called policies.
Section~\ref{sec:cernexp} has some example policies
on our illustrative data.
We include a counterfactual policy
with motivation similar to counterfactual XAI.
The relation between counterfactual XAI and choice of
background data are
also discussed in \citet{albini2021counterfactual} in detail.

A formal description of the curve
is as follows.
For a permutation
$(\pi(1),\pi(2),\dots,\pi(n))$ of $(1,2,\dots,n)$,
and $1\le j\le n$
define $\Pi(j)=\{\pi(1),\dots,\pi(j)\}$ with $\Pi(0)=\emptyset$.
Now let $\tilde \bsx^{(j)} = \bsx'_{\Pi(j)}{:}\bsx_{-\Pi(j)}$, that is
$$
    \tilde\bsx^{(j)}_k =
    \begin{cases}
        \bsx'_k, & k\in \Pi(j)  \\
        \bsx_k,  & \text{else.}
    \end{cases}
$$

For our theoretical study it is convenient to define
\begin{align}\label{eq:defauc}
    \auc = \sum_{j=0}^n f(\tilde\bsx^{(j)}).
\end{align}
If we connect the points $(j,f(\tilde\bsx^{(j)})$ by line segments
then the area we get is a sum of trapezoidal areas
$$
    \sum_{j=1}^n \frac12\bigl( f(\tilde\bsx^{(j-1)})
    +f(\tilde\bsx^{(j)})\bigr)
    =
    \auc -\frac12\bigl(
    f(\tilde\bsx^{(0)})+f(\tilde\bsx^{(n)})\bigr).
$$
The difference between this trapezoidal area and~\eqref{eq:defauc} is unaffected by the ordering permutation $\pi$
because $\tilde\bsx^{(0)}=\bsx$ and $\tilde\bsx^{(n)}=\bsx'$ are invariant
to the permutation.
One could similarly omit either $j=0$ or $j=n$
(or both) from the sum in~\eqref{eq:defauc}
without changing the difference between areas
attributed to any two permutations.

Our primary measure is the area below
the curve but above a straight line from
$(0,f(\tilde\bsx^{(0)}))$ to $(n,f(\tilde\bsx^{(n)}))$.
It is the (signed) area between those curves, that is
\begin{align}\label{eq:abc}
    \abc & = \auc - \aul,\quad\text{where} \\
    \aul & =
    \frac{n+1}2\bigl( f(\tilde\bsx^{(0)}) +f(\tilde\bsx^{(n)})\bigr)
    =\frac{n+1}2\bigl( f(\bsx) +f(\bsx')\bigr)
    \notag
\end{align}
is a measure of the area under the straight line
connecting $(0,f(\bsx))$ to $(n,f(\bsx'))$
compatible with our AUC formula from~\eqref{eq:defauc}.
The difference between ABC and AUC is also unaffected by the ordering of variables.
The AUC and ABC going from $\bsx$ to $\bsx'$ is the same as
that from $\bsx'$ to $\bsx$.  That is, it only depends on the
two selected points. The reason for this is that
$\tilde\bsx^{(k)}$ going from $\bsx'$ to $\bsx$ equals
$\tilde\bsx^{(n-k)}$ when we go from $\bsx$ to $\bsx'$.
For the same reason the deletion areas are
the same in both directions, but generally not
equal to their insertion counterparts.

Both AUC and ABC have the same units that $f$ has.
Then for instance if $f$ is measured in dollars
then $\abc/n$ is in dollars explained per feature.
This normalization is different from that of \cite{petsiuk2018rise} whose curve is over the interval $[0,1]$ instead of $[0,n]$ and whose AUC is then interpreted in terms of a proportion of pixels.

\subsection{Additive Functions}
The AUC measurements above are straightforward
to interpret when $f(\bsx)$ takes the
additive form
$f_\emptyset+\sum_{j=1}^nf_j(x_j)$
for a constant $f_\emptyset$ and
functions $f_j:\cx_j\to\real$.
We then easily find that
\begin{align}
    \auc & = (n+1)f(\bsx) + \sum_{j=1}^n(n-j+1)\bigl( f_j(x'_j)-f_j(x_j)\bigr),\quad\text{and}\label{eq:aucformulaadd} \\
    \abc & = \sum_{j=1}^n\Bigl(\frac{n+1}2-j\Bigr)\bigl( f_j(x'_j)-f_j(x_j)\bigr).\label{eq:abcformulaadd}
\end{align}
The best ordering is, unsurprisingly, the one that
sorts $j$ in decreasing values of $f_{j}(x_j)-f_j(x'_j)$.
If $f$ is additive then the insertion and deletion ABCs
are the same.
Also, the Shapley value for variable $j$ is proportional
to $f_j(x'_j)-f_j(x_j)$ and so ordering variables by
decreasing Shapley value maximizes the AUC and ABC.

\subsection{Interactions}
The effect of interactions is more complicated,
but we only need interactions involving points
$\bsx$ and $\bsx'$.  We define the differences and
iterated differences of differences that we need via
\begin{align*}
    \Delta_j & = \Delta_j(\bsx,\bsx',f)= f(\bsx'_j{:}\bsx_{-j})-f(\bsx),\quad\text{and} \\
    \Delta_u & = \Delta_u(\bsx,\bsx',f) = \sum_{v\subseteq u}(-1)^{|u-v|}
    f(\bsx'_v{:}\bsx_{-v})
\end{align*}
for $u\subseteq1{:}d$ with $\Delta_\emptyset = f(\bsx)$
corresponding to no differencing at all.
From Theorem~\ref{thm:aucfromdeltaf} of Appendix~\ref{sec:auctheory} we get
\begin{align*}
    \auc & = \sum_{u\subseteq1{:}n}\bigl( n-\lceil u\rceil+1\bigr)\Delta_u.
\end{align*}
We can interpret this as follows: the interaction
for variables $u$
when represented as differences of differences,
takes effect in its entirety
once the last element $j$ of $u$ has been changed
from $x_j$ to $x'_j$.  It then contributes
to $n-\lceil u\rceil+1$ of the summands.
Thus, in addition to ordering the
main effects from largest to smallest,
the quality score for a permutation takes account of
where large positive and large negative interactions
are placed.

It is easy to see from~\eqref{eq:abcformulaadd}
that for an additive function and a uniformly
random permutation $\pi$ we have
$\e(\abc)=0$ because under such random
sampling the expected rank of variable $j$ is $(n+1)/2$.

Now suppose that we permute the variables
$1$ through $n$ into a random permutation $\pi$.
A fixed subset $u=(j_1,\dots,j_{|u|})$ is then mapped
to $\pi(u) = (\pi(j_1),\dots,\pi(j_{|u|}))$.  Under this randomization
$$
    \e( \auc) = \sum_{u\subseteq1{:}n}
    \bigl( n -\e( \lceil \pi(u)\rceil)+1\bigr)\Delta_uf.
$$
The next proposition gives $\e(\lceil u\rceil)$.

\begin{proposition}\label{prop:meanceiling}
    For $n\ge1$ and $u\subseteq1{:}n$ let $\pi(u)$
    be a simple random sample of $|u|$ elements from $1{:}n$.
    Then
    \begin{align}\label{eq:meanceling}\e( \lceil \pi(u)\rceil) = \frac{|u|(n+1)}{|u|+1}.
    \end{align}
\end{proposition}
\begin{proof}
    The result holds trivially for $|u|=0$, so we suppose
    that $|u|\ge 1$.
    For $k\in\{|u|,\dots,n\}$,
    $\Pr( \lceil\pi(u)\rceil \le  k)  =
        {{k\choose |u|}}/{{n\choose |u|}}$
    because there are ${n\choose |u|}$ equally
    probable ways to select the elements of $\pi(u)$
    and ${ k \choose |u|}$ of those have $\lceil \pi(u)\rceil \le k$.
    Subtracting $\Pr( \lceil\pi(u)\rceil \le k-1)$ we get
    $$\Pr( \lceil \pi(u)\rceil = k) =
        {{k-1\choose |u|-1}}/{{n\choose |u|}
        =\frac{|u|}k{k\choose |u|}}/{{n\choose |u|}}.$$
    Then using the hockey stick identity,
    \begin{align*}
        \e(\lceil \pi(u)\rceil)
         & = |u| {n\choose |u|}^{-1}
        \sum_{k=|u|}^n{k\choose |u|}
        = u {n+1\choose |u|+1}\bigm/{n\choose |u|}
        =\frac{|u|(n+1)}{|u|+1}.\quad\qedhere
    \end{align*}
\end{proof}

Next we work out the expected
value of $\auc$. Using the decomposition in Appendix~\ref{sec:abc4del}
we have
$$
    \aul = \frac{n+1}2\bigl( f(\bsx)+f(\bsx')\bigr)
    = (n+1)\Delta_\emptyset + \frac{n+1}2\sum_{u\ne\emptyset}\Delta_u.
$$
Then with Proposition~\ref{prop:meanceiling}
we find that
\begin{align*}
    \e(\abc) & =
    \e(\auc-\aul)                                           \\
             & = \sum_{u\ne\emptyset}
    \Bigl(\frac{n+1}2-\e(\lceil\pi(u)\rceil)\Bigr)\Delta_uf \\
             & = \sum_{u\ne\emptyset}
    \Bigl(\frac{n+1}2-\frac{|u|(n+1)}{|u|+1}\Bigr)\Delta_uf \\
             & = \frac{n+1}2
    \sum_{u\ne\emptyset}
    \frac{1-|u|}{|u|+1}
    \Delta_uf.
\end{align*}

As noted above, $\e(\abc)=0$ if $f$ has no interactions
because $\e(\lceil \{j\}\rceil)=(n+1)/2$,
but otherwise it need not be zero
because $\e(\lceil u\rceil)>(n+1)/2$ for $|u|>1$.
The contribution to $\abc$ from a given interaction
has the opposite sign of that interaction
because $|u|-1<0$ for $|u|\ge2$.
We show in Appendix~\ref{sec:abc4del} that
$\e(\abc+\abc')=0$ under random permutation of the
indices.

\subsection{AUC Versus Shapley}\label{sec:aucvsshapley}

If we order variables in decreasing order by Shapley value,
that does not necessarily maximize the AUC.  We can see this
in a simple setup for $n=3$ by constructing certain
values of $\Delta_u$.  We will exploit the delay $\lceil u\rceil$
with which an interaction gets `credited' to an AUC
to find our example.

Consider a setting with $n=3$ and $\Delta_1=3$, $\Delta_2=2$, $\Delta_3=1$, $\Delta_{\{1,2\}}=A$ to be chosen later
and all other $\Delta_u=0$. Because the $\Delta_u$ values are
also Harsanyi dividends \citep{hars:1959}, the Shapley value shares them equally among
their members.  Therefore
\begin{align*}
    \phi_1 = 3 + A/2,\quad\phi_2=2+A/2\quad\text{and}\quad\phi_3=1.
\end{align*}
So long as $A>-2$, the Shapley values are ordered
$\phi_1>\phi_2>\phi_3$. The $\auc$ for this ordering is
$$
    \auc((1,2,3)) = 3\Delta_1+2\Delta_2+\Delta_1+(3-\lceil\{1,2\}\rceil+1)A
    =14+2A.
$$
If we order the variables $(1,3,2)$ then the AUC is
$$
    \auc((1,3,2)) = 3\Delta_1+2\Delta_3+\Delta_2+(3-\lceil\{1,3\}\rceil+1)A
    =13 + A.
$$
As a result we see that if $-2<A<-1$, then
$$
    \phi_2>\phi_3\quad\text{but}\quad \auc((1,3,2))>\auc((1,2,3)).
$$
It is easy to show
that there can be no counterexamples for $n=2$.

\section{Incorporating Binary Features into Integrated Gradients}\label{sec:ig4binary}

IG avoids the exponential computational costs
that arise for Shapley value.
However, as defined
it is only available for variables with continuous
values.  Many problems have binary variables and so
we describe here some approaches to including them.

IG is based on the Aumann-Shapley value
from \cite{auma:shap:1974} who present
an axiomatic derivation for it.
We omit those axioms.
See \cite{sundararajan2017axiomatic} and \cite{lundberg2017unified} for the axioms
in a variable importance context.

We orient our discussion around
Figure~\ref{fig:category_shap} that
shows a setting with 3 variables.
Panel (a) shows a target data point that
differs from a baseline point in three
coordinates.
Panel (b) shows the diagonal path taken
by integrated gradients.  The gradient of
$f(\bsx)$ is integrated along that path
to get the IG attributions:
$$ A_f(\bsx,\bsx') \equiv \int_0^1
    \nabla f(\bsx+t(\bsx'-\bsx))\rd t\in\real^n.$$
If one of the variables is binary then one approach, shown
in panel (c) is to simply jump from one value to another at some intermediate point, such as the midpoint.  For differentiable $f$ the integral of the gradient along the line segment given by the jump would, by the fundamental theorem of calculus, be the difference between $f$ at the ends of that interval (times the Euclidean basis vector
$(0,\dots,0,1,0,\dots,0)$ corresponding to that variable). Such a difference is computable for binary variables even though the points on the path are ill defined. Finally, the vector of Shapley values is an average over $n!$ paths making jumps from baseline to target in all possible variable orders \citep{sundararajan2017axiomatic}. Panel (d) shows two of those paths.  For differentiable $f$ one could integrate gradients along those
paths as a way to compute the jumps, again by the
fundamental theorem of calculus, and then average
the path integrals.

Now suppose that we have $m>0$ binary variables
in a set $v\subset1{:}n$. Without loss of generality
suppose that $v=1{:}m$.
Then any data $\bsx$ and $\bsx'$ are in $\{0,1\}^m\times\real^{n-m}$
and for IG we need to consider arguments to $f$ in $\real^n$.
We consider three choices that we describe
in more detail below:
\begin{compactenum}[\quad\bf a)]
    \item Use the fitted $f$ as if the binary $x_j\in[0,1]$, for $j\in 1{:}m$.
    \item Replace $f$ by a multilinear interpolation
    $$g(\bsx) = \sum_{u\subseteq1{:}m}f(\bsone_u{:}\bszero_{1:m-u}{:}\bsx_{(m+1):n})\prod_{j\in u}x_j\prod_{j\in 1:m-u}(1-x_j)$$
    and compute the integrated gradients of $g$.
    \item Take paths that jump for binary $x_j$ as shown in Figure~\ref{fig:category_shap}(c).
\end{compactenum}
We call these choices, casting, interpolating and jumping, respectively.

Option {\bf a} is commonly available as many
machine learning models cast binary variables
to real values when fitting.
Option {\bf b} interpolates: if $\bsx_{1:m}\in\{0,1\}^m$
then $g(\bsx)=f(\bsx)$.  \cite{sundararajan2017axiomatic} show that integrated gradients match Shapley values for functions
that are a sum of a multilinear interpolation like $g$
above plus a differentiable additive function.
Unfortunately, the cost of evaluating $g(\bsx)$ is $\Omega(2^m)$, which is exponential in
the number of binary inputs.
For option {\bf c} we have to choose where to make
the $m$ jumps.  For $m=1$ we would naturally jump half
way along the path
though there is not an axiomatic reason for that choice.
For $m>1$ we have to choose $m$ points on the curve
at which to jump. Even if we decide that all of those
jumps should be at the midpoint, we are left with $m!$
possible orders in which to make those jumps.  By symmetry
we might want to average over those orders but that produces
a cost which is exponential in $m$.

Based on the above considerations we think that
the best way to apply IG to binary variables
is also the simplest.  We cast the corresponding
booleans to floats.

\begin{figure}[t]
    \centering
    \includegraphics[width=\linewidth]{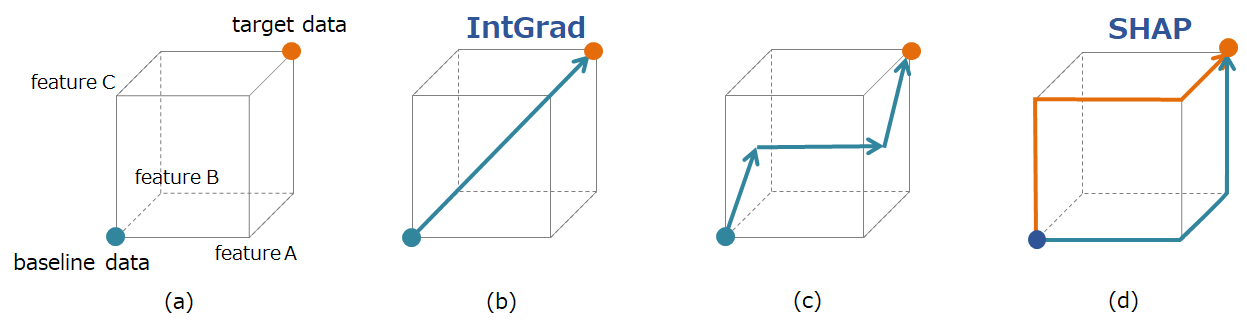}
    \caption{(a): $\bsx$ and $\bsx'$ placed in feature space represented as a cube,
        (b): usual Integrated Gradients with a straight line path connecting $x$ and $x'$,
        (c): Integrating along a path with one jump,
        (d): Shapley value expressed as an average of path integrals over $n!$ paths.}
    \label{fig:category_shap}
\end{figure}

\section{Experimental Results}
\label{sec:exp}

In this section we illustrate insertion and deletion
tests for regression.  We compare
our variance importance measures
on two tabular datasets.  The first one is about predicting the
value of houses in India.  It has mostly binary
predictors and two continuous ones.
The second dataset is from CERN and it computes
the invariant mass produced from some electron
collisions.

The methods we compare are Kernel SHAP \citep{lundberg2017unified},
integrated gradients \citep{sundararajan2017axiomatic},
DeepLIFT \citep{shrikumar2017learning},
Vanilla Grad \citep{simonyan2013deep},
Input$\times$Gradient \citep{shrikumar2016not}
and LIME \citep{ribeiro2016should}.
\cite{ancona2017towards} make a detailed study of backpropagation-based gradient methods including all of the above ones (this excludes LIME) as well as layer-wise relevance propagation (LRP) of \cite{bach2015pixel} that we did not include in our computations.

\subsection{Bangalore Housing Data}
\label{sec:expindia}

The dataset we use lists the value in Indian rupees (INR) for houses in India. The data are from Kaggle at this URL:\\
\url{www.kaggle.com/ruchi798/housing-prices-in-metropolitan-areas-of-india}.\\
We use 38 of the 39 other columns in this dataset
(excluding the ``Location" column that contains various place names
for simplicity),
and we treat the ``Area" and ``No.\ of Bedrooms" columns as continuous variables.

We use only the data from Bangalore (6,207 records).  Most of those data points were missing almost every predictor.  We use only 1,591 complete data points. We normalized the output value by dividing by 10,000,000 INR. We centered the continuous variables at their means and then divided them by their standard deviations.
We selected 80\% of the data points at random to train a multilayer perceptron (MLP).
The hyperparameters such as number of layers and ratio of dropouts
are determined from a search
described in Appendix~\ref{sec:housingprices}.

For the 20\% of points that were held out (391 observations) we computed
the ABC from~\eqref{eq:abc} using
our collection of variable importance methods.
We also included
a random variable ordering as a check.
For each of those points $\bsx$ we made a careful selection of
a reference point $\bsx'$ from holdout points as follows:
\begin{compactitem}
    \item the point $\bsx'$ had to differ from $\bsx$ in at least $12$ features,
    \item it had to be among the smallest $20$ such values of $\Vert\bsx-\cdot\Vert$, and
    \item among those $20$ it had to have the greatest response difference $|f(\bsx)-f(\bsx')|$.
\end{compactitem}
Having numerous different features makes the attribution problem more challenging.  Having $\Vert\bsx-\bsx'\Vert$ small brings less exposure to
the problems of unrealistic data hybrids.  Finally, having large 
$|f(\bsx)-f(\bsx')|$, the absolute value of the sum of feature attributions for XAI algorithms with the completeness axiom,
identifies data pairs in most need of an attribution.  Despite having close feature vectors those pairs have quite different predicted responses.

Our implementation of KS used 120,000 samples.
Our implementation of IG used a Riemann sum on 500
points to approximate the line integrals.
The hyperparameters for other XAI methods are
summarized in Appendix~\ref{sec:otherxais}.

The ABCs and their differences are summarized in Table \ref{tb:diffABC}.
There we see numerically that KS was best
for the insertion ABC and IG was second best.
For the deletion measure it was essentially
a three way tie for best among KS, IG and DeepLift.

The simple gradient based methods were disappointing.
In particular, vanilla grad was worse than random. We note that vanilla grad uses a default variable scaling.  We used the standard deviation of each input while another choice is to scale each variable to the interval $[0,1]$.  Neither of these choices use the specific baseline-target pair and this could cause poor performance.

The difference between KS and IG was not
very large.
Thus even in this setting where there are lots of binary
predictors,
IG was able to closely mimic KS.
We see in Table~\ref{tb:diffABC} that
insertion ABCs are on average higher than deletion ABCs
for this policy.

    {While KS and IG and
        LIME and DeepLIFT make use of
        reference values in computation of feature attributions,
        vanilla grad and Input$\times$Gradient
        do not require one to specify reference values.
        They are determined only by local information
        around the target data.
        Since our ABC criterion is defined
        in terms of reference values it is not surprising
        that methods which use those reference values
        get larger ABC values.  It is interesting that a method like Input$\times$Gradient that does not even know the baseline we compare to can do
        as well as it does here.
    }
We note that DeepLIFT does
reasonably well compared to KS and IG,
even though DeepLIFT
is derived without any axiomatic properties such as
the Aumann-Shapley axioms.
\citet{ancona2017towards} pointed out
a connection wherein
DeepLIFT can be interpreted as the approximation of
a Riemann sum of IG by a single step with average value
in spite of the difference in their computational procedures.
Their implementation details are also summarized in Appendix  \ref{sec:otherxais}.

{
KS performed well and IG is a fast approximation
to it.  Therefore we compare the ABC of insertion and
deletion tests for KS and IG
in Figure \ref{fig:aggregated_insertion}.}
In both cases the left panel shows that the
ABC for KS has a long tail.
This is also true for IG, but to save space we
omit that histogram. Instead we show in the middle panels
that the ABC for IG is almost the same as that for KS
point by point, with KS usually attaining a somewhat
better (larger) ABC than IG.  The right panels there show that ABCs
for random orderings have nearly symmetric distributions with insertion tests having a few more outliers than
deletion tests do.

\begin{table}
    \centering
    \begin{tabular}{llcc}
        \toprule
        Test Mode & Method                               & $\phm$Mean  & Std. Error \\
        \midrule
        Insertion & Kernel SHAP                          & $\phm$0.628 & 0.034      \\
                  & Integrated Gradients                 & $\phm$0.572 & 0.033      \\
                  & DeepLIFT                             & $\phm$0.548 & 0.032      \\
                  & Vanilla Grad                         & $-$0.093    & 0.026      \\
                  & Input$\times$Gradient                & $\phm$0.206 & 0.027      \\
                  & LIME                                 & $\phm$0.499 & 0.028      \\
                  & Random                               & $-$0.020    & 0.023      \\
                  & Kernel SHAP $-$ Integrated Gradients & $\phm$0.057 & 0.007      \\
        \midrule
        Deletion  & Kernel SHAP                          & $\phm$0.423 & 0.032      \\
                  & Integrated Gradients                 & $\phm$0.422 & 0.032      \\
                  & DeepLIFT                             & $\phm$0.425 & 0.031      \\
                  & Vanilla Grad                         & $-$0.098    & 0.029      \\
                  & Input$\times$Gradient                & $\phm$0.185 & 0.029      \\
                  & LIME                                 & $\phm$0.395 & 0.032      \\
                  & Random                               & $-$0.023    & 0.012      \\
                  & Kernel SHAP $-$ Integrated Gradients & $\phm$0.001 & 0.003      \\
        \bottomrule
    \end{tabular}
    \caption{Mean insertion and deletion ABCs for 391 of the Bangalore
        housing data points, rounded to three places. }
    \label{tb:diffABC}
\end{table}

\begin{figure}[t]
    \centering
    \includegraphics[width=\linewidth]{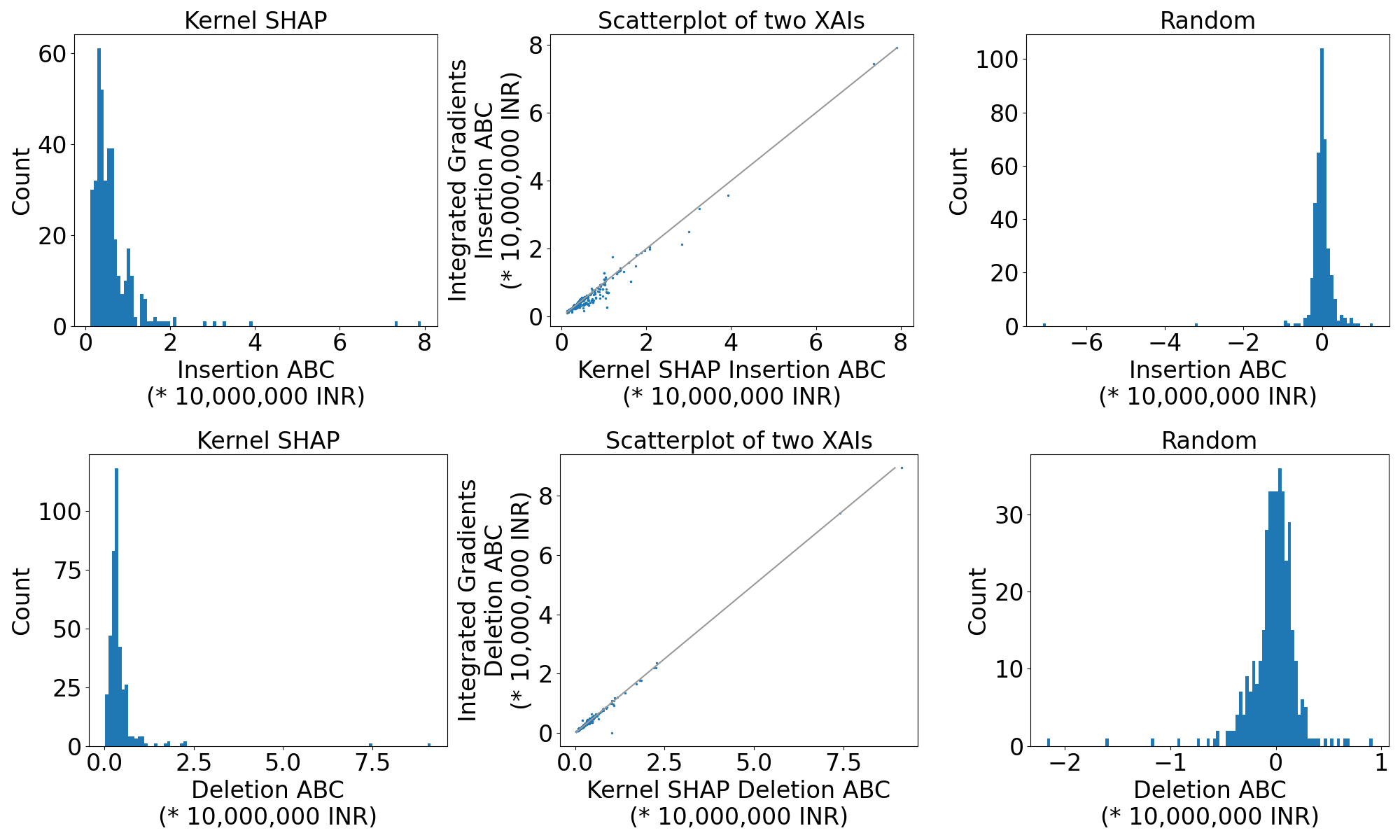}
    \caption{
        The top row shows results of insertion tests.  The bottom row is for deletion tests.
    }
    \label{fig:aggregated_insertion}
\end{figure}

Figure \ref{fig:tiled_inlier} shows some insertion
and deletion curves comparing a randomly
chosen data point to a counterfactual reference point.
Figure \ref{fig:tiled_outlier}
shows analogous plots for
the data with the greatest
differences in ABC between
KS and IG.
It shows that  a very
large ABC difference between methods in the insertion test need not have a large difference in the deletion tests and vice versa.

\begin{figure}[t]
    \includegraphics[width=\linewidth]{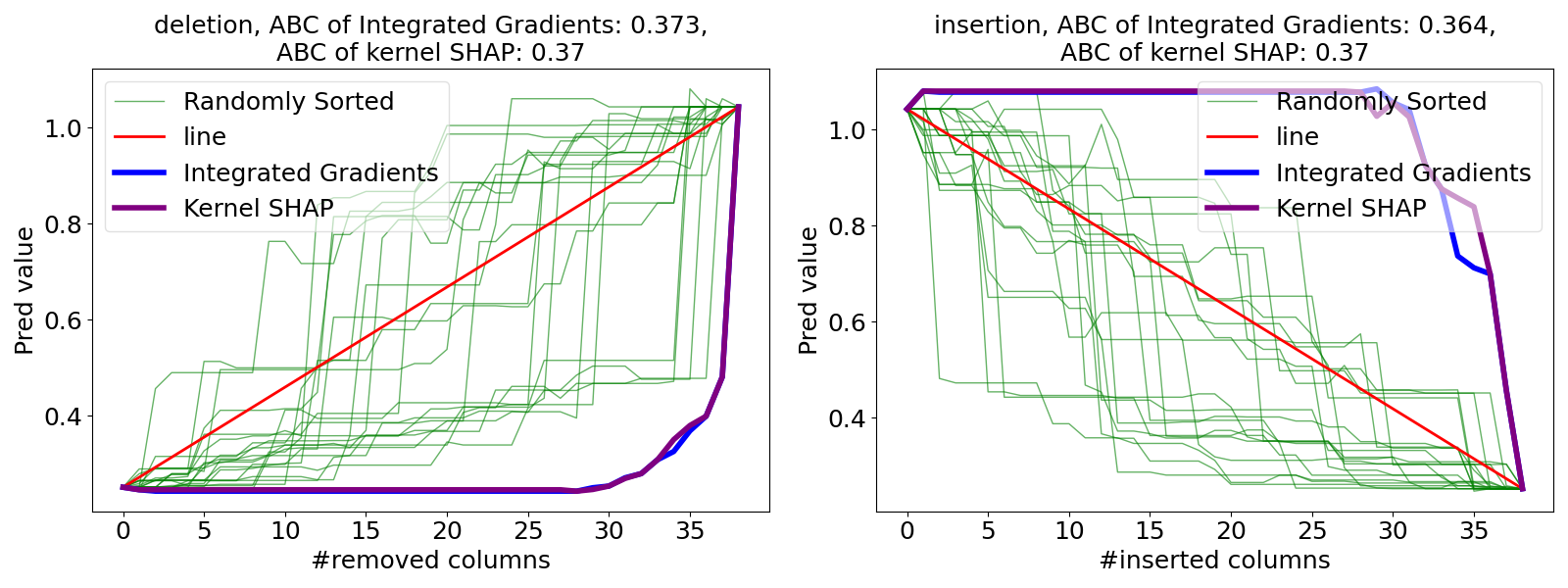}
    \caption{An example of the deletion and insertion tests for the Bangalore housing  dataset.
        The 
        left panel is for a deletion test and the
        right panel is for a insertion test.
        Both pictures include 20 curves for
        random variable orders. 
    }
    \label{fig:tiled_inlier}
\end{figure}

\begin{figure}[t]
    \includegraphics[width=\linewidth]{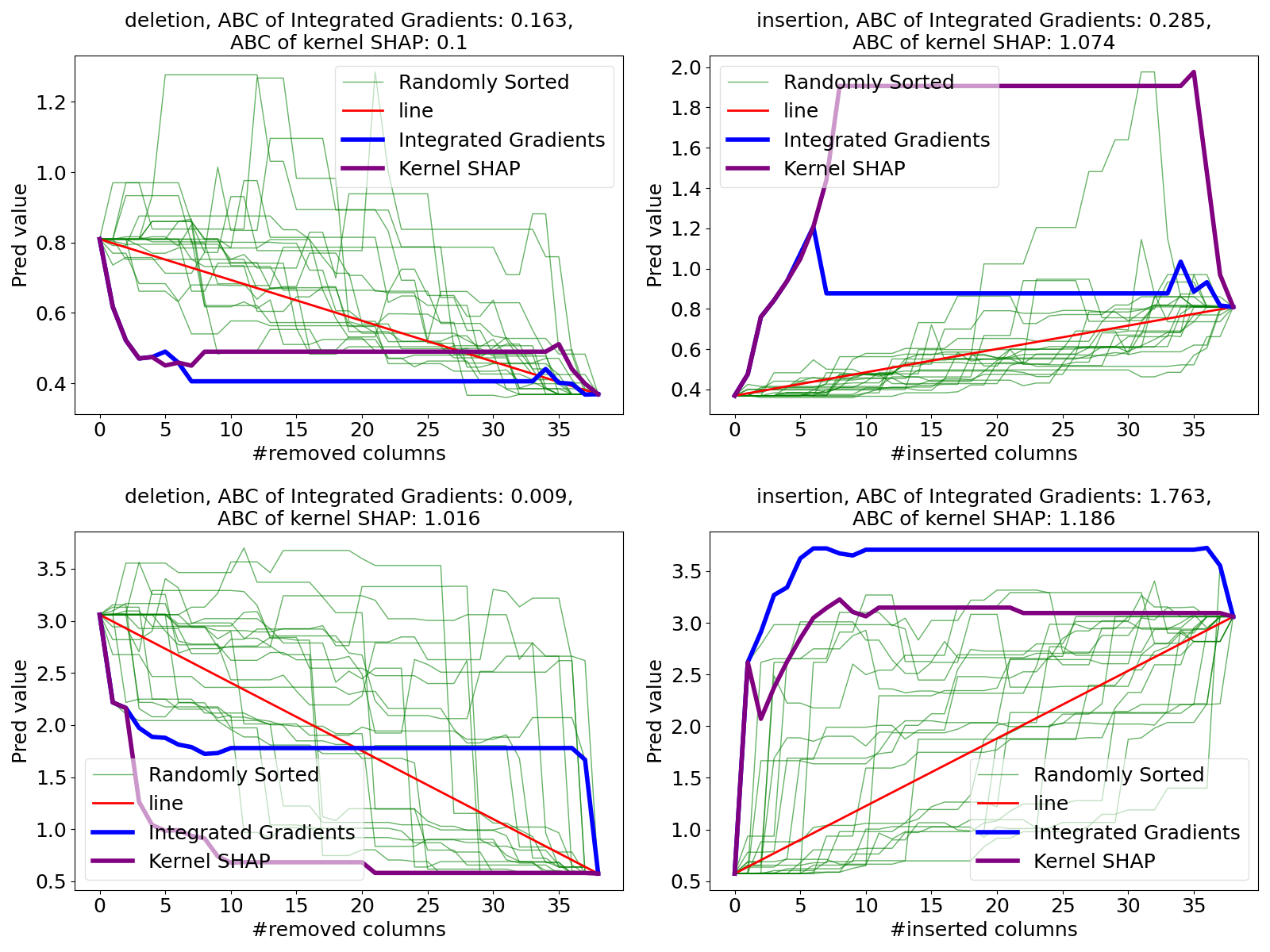}
    \caption{These figures show data
        points whose differences in ABC between
        KS and IG are largest in each of tests in
        the Bangalore housing  dataset.
    }
    \label{fig:tiled_outlier}
\end{figure}

\subsection{CERN Electron Collision Data}
\label{sec:cernexp}
The CERN Electron Collision Data
\citep{cerndataset}
is a dataset about dielectron collision events at CERN.
It includes continuous variables
representing the momenta and energy of the electrons,
as well as discrete variables for the charges of the electrons ($\pm 1$: positrons or electrons).
Only the data whose invariant mass of two
electrons (or positrons) was in the range from $2$ to $110$
GeV were collected.
We treat it as a regression problem to predict
their invariant mass from the other 16 features.

The data contains the physical observables of two electrons after the collisions
whose tracks are reconstructed from the information captured in detectors
around the beam.
The features are as follows:
The total energy of the two electrons,
the three directional momenta, the transverse momentum, the pseudorapidity,
the phi angle and the charge of each electron.
They are highly dependent features because some of them are calculated from the others.
For instance, since a beam line is aligned on the $z$-axis as usual in particle physics,
the transverse momenta $p_{t_i}$ for $i=1,2$ are composed of $p_{x_i}$ and $p_{y_i}$
such that $p_{t_i}^2 = p_{x_i}^2 + p_{y_i}^2$.
The phi angle $\phi_i$ is the angle between $p_{x_i}$ and $p_{y_i}$ where
$p_{t_i} = p_{x_i} \cos \phi_i$.
The total energy of them is also calculated relativistically.
Since the momenta are recorded in GeV unit,
which overwhelms the static mass of electrons ($\sim511$ keV in natural unit),
the total energies $E_{i}$ are $E_i^2 \simeq p_{x_i}^2 + p_{y_i}^2 + p_{z_i}^2$.
The pseudorapidities $\eta_i$ are given as
angles from a beam line.
The definition is
$\eta = -\frac{1}{2} \ln \frac{|\bm{p}|-p_z}{|\bm{p}|+p_z} \simeq -\frac{1}{2} \ln \frac{E-p_z}{E+p_z}$
for each $\eta_i$.
Regarding these definitions,
only 8 (three directional momenta and the charges of each electrons) of the 16 are
independent features,
and the other features are used as
convenient transformed coordinates in particle physics.
Actually, the invariant mass $M$, the prediction target feature, is also
approximated arithmetically
from the momenta as $M^2 \simeq (E_1 + E_2)^2 - |\bm{p}_1 + \bm{p}_2|^2$
within 0.6\% residual error on average for this dataset.
From this viewpoint,
users of machine learning might confirm that the models properly exclude the
charges from evidence of the predictions via XAI.
This aspect is, in a sense, a case where
ground truth in XAI can be obtained from domain knowledge. We have made such investigations but
omit them to save space and because they are
not directly related to insertion and deletion
testing.

We omit data with missing values, and randomly select 80\% of the complete
observations (79,931 data points) to construct an MLP.
Embedding layers are not placed in this MLP, and
all variables are Z-score normalized.
Hyperparameters such as the number of units in
each layer of the MLP are determined by a
hyperparameter search described in Appendix~\ref{sec:cerncolldata}.

The predictions for the  2,000 held out data points $\bsx$
were inspected with both KS and IG.
The reference data $\bsx'$ used in XAI methods are collected from these 2,000 data under this policy:
\begin{compactitem}
    \item $x'_j\ne x_j$ for all $j\in1{:}n$ including charges,
    \item $\bsx'$ is among the $20$ smallest such $\Vert\bsx-\cdot\Vert$ values, and
    \item it maximizes $|f(\bsx')-f(\bsx)|$ subject to the above.
\end{compactitem}
This policy is called the  counterfactual policy below.  It has similar
motivations to the policy we used for the Bangalore housing data.
In this case it was possible to compute KS exactly using $2^{16}=65{,}536$ function evaluations.

    {
        The results are given in Table~\ref{tb:diffABC_cern}.
        KS is best for both insertion and deletion measures.
        IG and DeepLIFT are close behind.  LIME is nearly
        as good and the simple gradient methods once again
        do poorly.  As we did for the Bangalore housing data, we
        make graphical comparisons between KS and IG.}

\begin{table}
    \centering
    \begin{tabular}{llrc}
        \toprule
        Test Mode & Methods                              & \hfil Mean & Std. Error \\         \midrule
        Insertion & Kernel SHAP                          & 18.535     & 0.215      \\
                  & Integrated Gradients                 & 18.289     & 0.213      \\
                  & DeepLIFT                             & 18.118     & 0.211      \\
                  & Vanilla Grad                         & $-$1.310   & 0.252      \\
                  & Input$\times$Gradient                & 7.620      & 0.216      \\
                  & LIME                                 & 17.319     & 0.209      \\
                  & Random                               & $-$0.380   & 0.175      \\
                  & Kernel SHAP $-$ Integrated Gradients & 0.246      & 0.025      \\
        \midrule
        Deletion  & Kernel SHAP                          & 16.752     & 0.176      \\
                  & Integrated Gradients                 & 16.315     & 0.173      \\
                  & DeepLIFT                             & 16.646     & 0.176      \\
                  & Vanilla Grad                         & 0.226      & 0.256      \\
                  & Input$\times$Gradient                & 7.940      & 0.187      \\
                  & LIME                                 & 15.845     & 0.170      \\
                  & Random                               & $-$0.268   & 0.179      \\
                  & Kernel SHAP $-$ Integrated Gradients & 0.437      & 0.025      \\
        \bottomrule
    \end{tabular}
    \caption{Mean insertion and deletion ABCs for 2000 CERN electron collision data points, under the counterfactual policy described in the text.}
    \label{tb:diffABC_cern}
\end{table}

{
The results for the insertion test are
shown in Figure \ref{fig:aggregated_insertion_cern}.
The deletion test results were very similar and
are omitted.
These results are similar to what we saw
in the previous experiment.
As a meaningful XAI metric,
KS provides a larger ABC than
the other orderings we tried.
Also, even in this case where differentiability
with respect to charges cannot be assumed,
IG does nearly as well as KS.
}

\begin{figure}[t]\begin{center}
        \includegraphics[width=\linewidth]{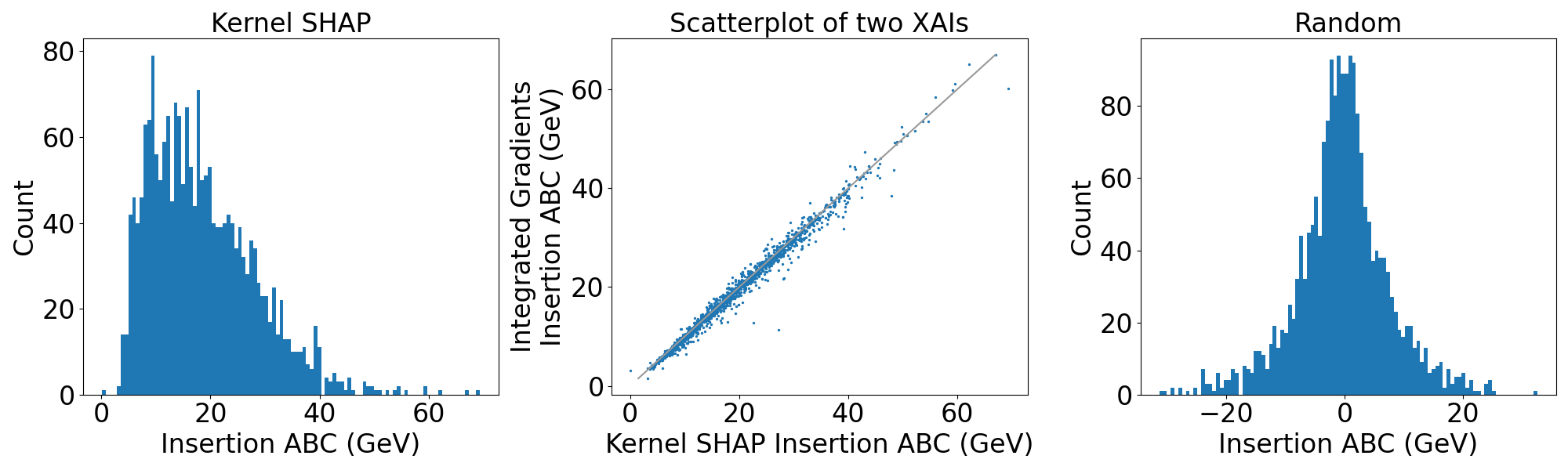}
        \caption{Results of the insertion test in CERN electron collision data. Results for the deletion test looked almost the same.}
        \label{fig:aggregated_insertion_cern}
    \end{center}\end{figure}

Although most of data are close to the 45 degree line in the
center panel of
Figure \ref{fig:aggregated_insertion_cern}
there are a few cases where KS gets a much  larger
ABC than IG does.
Two such points are shown in Figure \ref{fig:tiled_outlier_cern}.
Similarly to what we saw in the Bangalore housing example,
the comparisons where KS and IG differ greatly
in the insertion test has them similar in the
deletion test and vice versa.
A scatterplot of deletion versus insertion
areas is given in Figure \ref{fig:corr_insertiondeletion}.
Here and in the following we again only pick KS and IG as representative examples.
Most of the data are
near the diagonal in that plot
but there are some exceptional outlying points
where the two ABCs are quite different from each other.

\begin{figure}[t]
    \includegraphics[width=\linewidth]{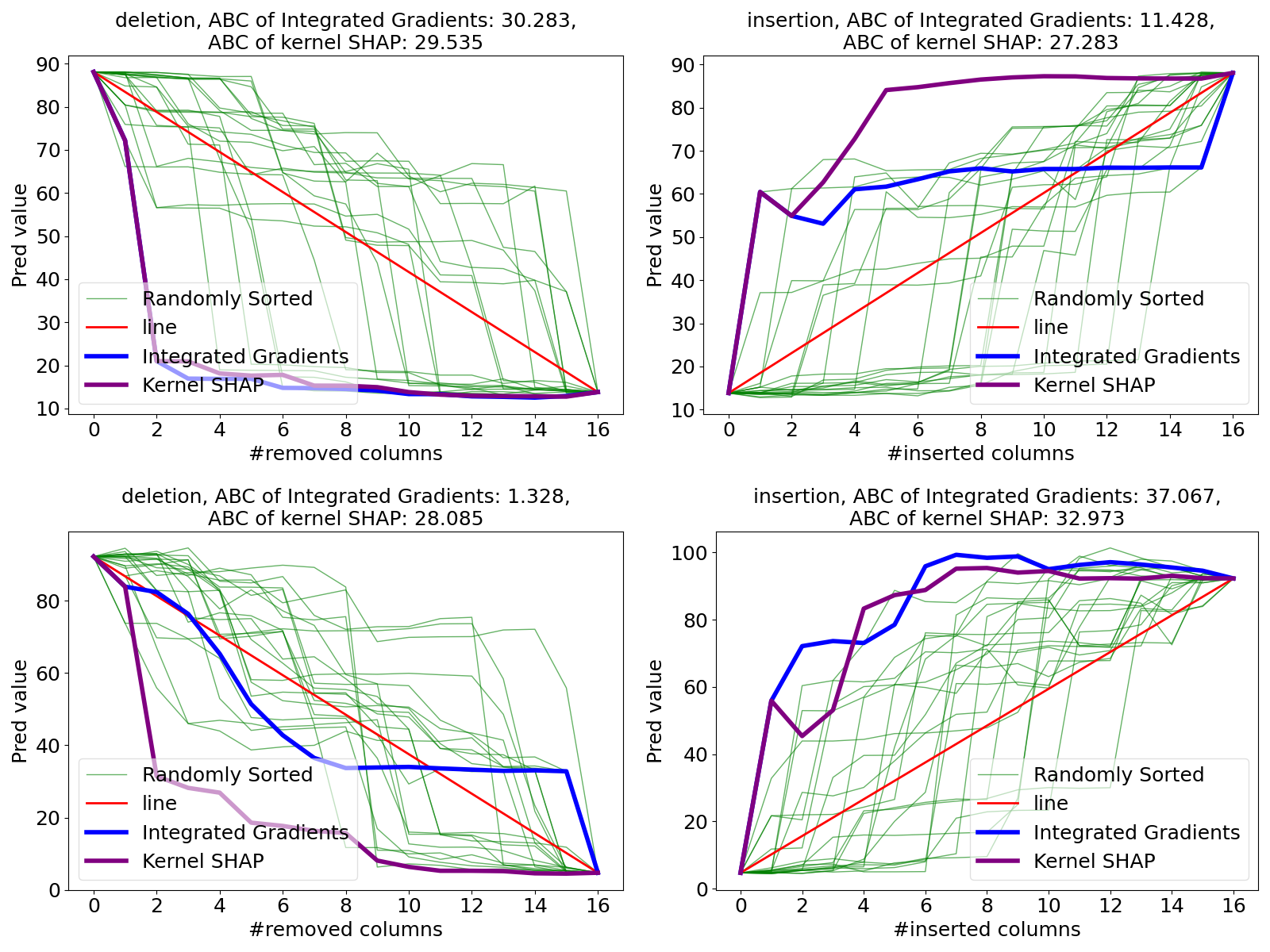}
    \caption{Deletion test outliers in the CERN electron collision data.
        The top row shows a comparison where KS had much
        greater ABC for insertion than IG had.  The bottom row shows a comparison
        where KS had much greater ABC for deletion. In both cases
        KS was comparable to IG for the other ABC type.}
    \label{fig:tiled_outlier_cern}
\end{figure}

\begin{figure}[t]
    \includegraphics[width=\linewidth]{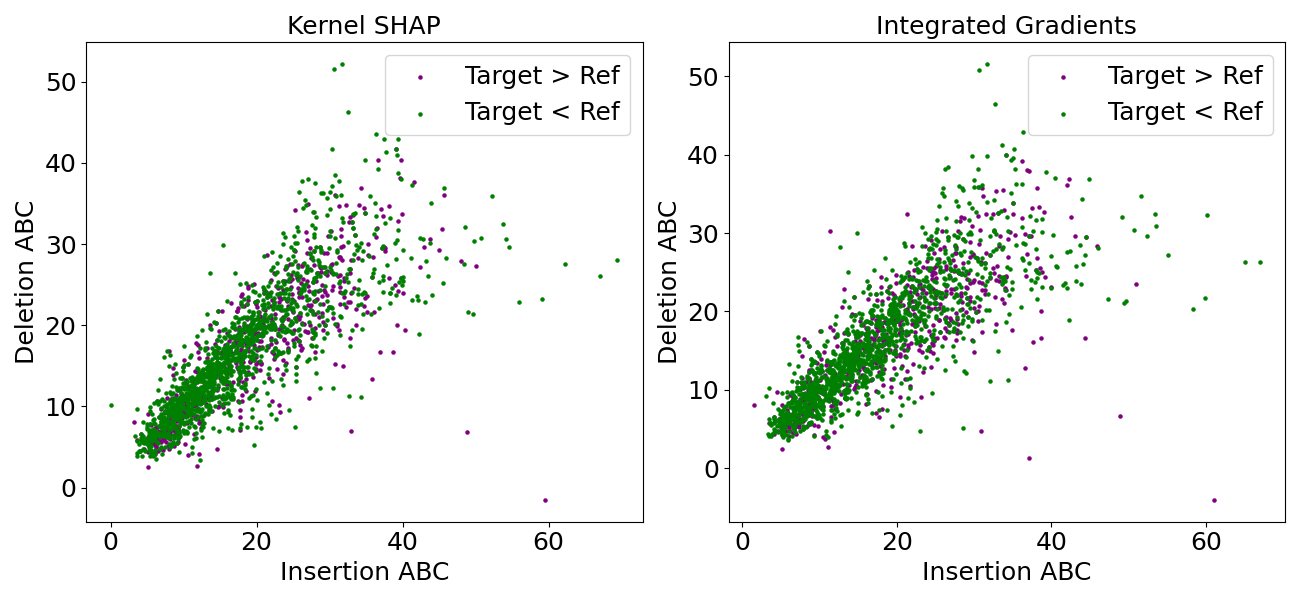}
    \caption{For the CERN data with the counterfactual policy, these
        figures show deletion versus insertion ABCs for
        KS (left) and IG (right).
        They are colored depending on relations of model outputs
        at target and reference point;
        Purple dots: $f(\bsx') > f(\bsx)$ and green dots: $f(\bsx') < f(\bsx)$.}
    \label{fig:corr_insertiondeletion}
\end{figure}

Next we consider two more policies, different
from our counterfactual policy.
In a `one-to-one policy', observations are paired up completely
at random.  They almost always get different values of
the continuous parameters and they get, on average one
different value among the two binary values.
We also consider an `average policy' where the reference point
has the average value for all features.  Such points
are necessarily unphysical, for instance, they have
near zero charge.

The results of these two policies are shown in
Table~\ref{tb:diffABC_oto}. KS attains the best
ABC value in all four comparisons there and IG
is always close but not always second, even
though it treats the particle charges as continuous
quantities.
One striking feature of that table is that the
insertion ABCs are much larger than the deletion
ABCs.  There is a simple explanation.   The invariant
masses must be positive and their distribution is
positively skewed.  The model never predicted a negative
value and the predictions also have positive skewness.
Because the predicted values satisfy a sharp lower
bound, that reduces the maximum possible deletion ABC. Because
they are positively skewed, higher insertion ABCs
are possible.  We see LIME does very well on the one-to-one
policy examples as it did on the counterfactual
policy, but it does not do well on the (unphysical) average
policy comparisons.

\begin{table}
    \centering
    \begin{tabular}{llrcrc}
        \toprule
                  &                       & \multicolumn{2}{c}{One-to-One} & \multicolumn{2}{c}{Average}                             \\
        Test Mode & Method                & $\phm$Mean                     & Std. Error                  & $\phm$Mean   & Std. Error \\
        \midrule
        Insertion & Kernel SHAP           & $\phm$30.392                   & 0.535                       & $\phm$23.518 & 0.216      \\
                  & Integrated Gradients  & $\phm$28.430                   & 0.515                       & $\phm$21.549 & 0.251      \\
                  & DeepLIFT              & $\phm$26.740                   & 0.433                       & $\phm$21.333 & 0.238      \\
                  & Vanilla Grad          & $\phm$1.707                    & 0.215                       & $\phm$4.803  & 0.194      \\
                  & Input$\times$Gradient & $\phm$14.677                   & 0.397                       & $\phm$21.610 & 0.236      \\
                  & LIME                  & $\phm$27.008                   & 0.510                       & $\phm$11.892 & 0.221      \\
                  & Random                & $\phm$3.795                    & 0.266                       & $\phm$4.242  & 0.156      \\
        \midrule
        Deletion  & Kernel SHAP           & $\phm$11.806                   & 0.299                       & $\phm$7.621  & 0.149      \\
                  & Integrated Gradients  & $\phm$10.643                   & 0.310                       & $\phm$6.856  & 0.128      \\
                  & DeepLIFT              & $\phm$11.485                   & 0.294                       & $\phm$7.305  & 0.134      \\
                  & Vanilla Grad          & $-$5.571                       & 0.323                       & $-$4.025     & 0.135      \\
                  & Input$\times$Gradient & $\phm$1.839                    & 0.296                       & $\phm$5.957  & 0.139      \\
                  & LIME                  & $\phm$10.877                   & 0.269                       & $\phm$1.420  & 0.160      \\
                  & Random                & $-$3.797                       & 0.301                       & $-$4.323     & 0.157      \\
        \bottomrule
    \end{tabular}
    \caption{Mean insertion and deletion ABCs for 2,000 of the CERN Electron Collision Data
        points whose reference data are determined in one-to-one policy and average policy respectively.
        The figures are rounded to three places. }
    \label{tb:diffABC_oto}
\end{table}

Figure \ref{fig:corr_insertiondeletion_oto}
shows results for the one-to-one policy.
Unlike Figure~\ref{fig:corr_insertiondeletion} for the counterfactual policy the points for $f(\bsx)>f(\bsx')$ are exactly the same as those that had $f(\bsx')>f(\bsx)$.

One data pair attaining an extreme
difference in Figure
\ref{fig:corr_insertiondeletion_oto} is inspected in
Figure \ref{fig:oto_asym}.
This data gains over 150 GeV in the insertion ABC,
which is anomalously large as this dataset is composed of
events with invariant mass between $2$ and $110$ GeV.
The figure shows that this large output is
due to an extraordinary response to artificial data
which appear on the path connecting $\bsx'$ and $\bsx$.
Both XAI methods (IG and KS) correctly identify
the features that bring on this effect for the insertion test.
Since the one-to-one policy can make long distance pairs
compared to the counterfactual policy,
synthetic data in the insertion and deletion processes
can be far from the data manifold.

The result for the average policy where the reference data is common to all test data is shown in Figure \ref{fig:corr_insertiondeletion_ave}.
This is a setting where the reference data are very far out of distribution, just like a single color image is for an
image classification task. In this case the deletion test, replacing real pixels in order by average ones, attains much smaller ABC values than the insertion test that starts with the average values.
From the above results in Figure \ref{fig:corr_insertiondeletion}, \ref{fig:corr_insertiondeletion_oto} and \ref{fig:corr_insertiondeletion_ave} it is easy to observe that
the relational distributions of insertion and deletion tests are totally different
depending on the reference data policy, even for the same XAI algorithm.
One cause is asymmetry in the response: if $f(\bsx)$ is sharply bounded
below but not above then insertion ABCs can be much larger than
deletion ones.

\begin{figure}[t]
    \includegraphics[width=\linewidth]{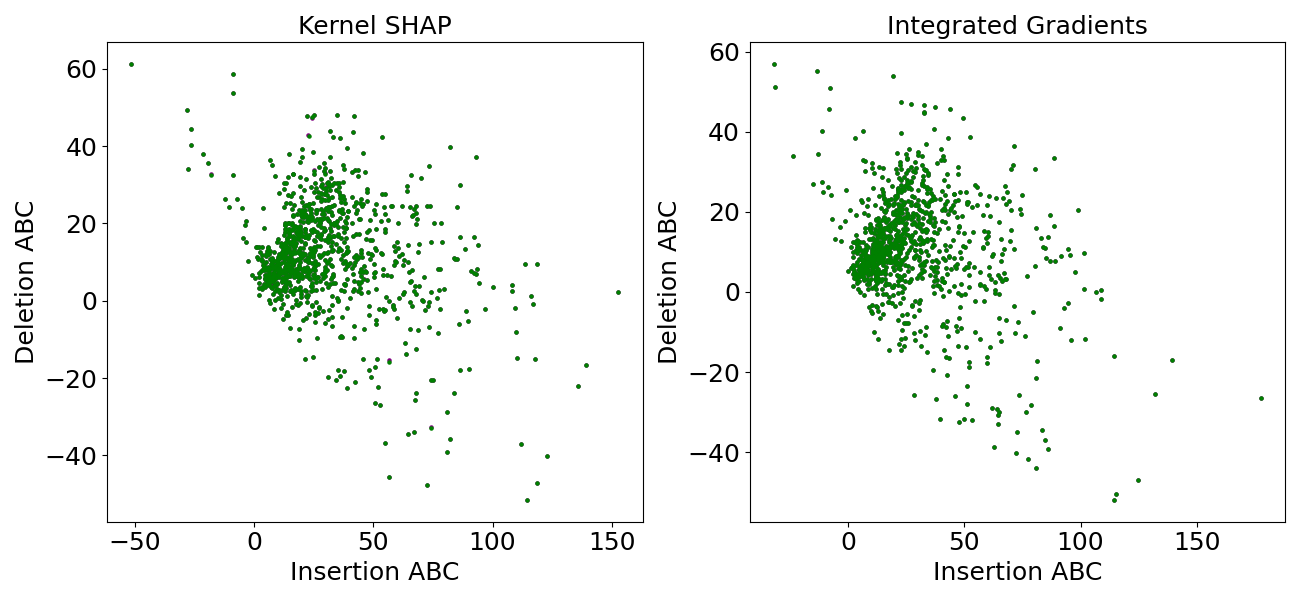}
    \caption{For the CERN data with the one-to-one policy, these
        figures plot deletion versus insertion ABCs.  The left plot is for KS and the right is for IG.
        Each dot correspond to two data that configure a pair.}
    \label{fig:corr_insertiondeletion_oto}
\end{figure}

\begin{figure}[t]
    \includegraphics[width=\linewidth]{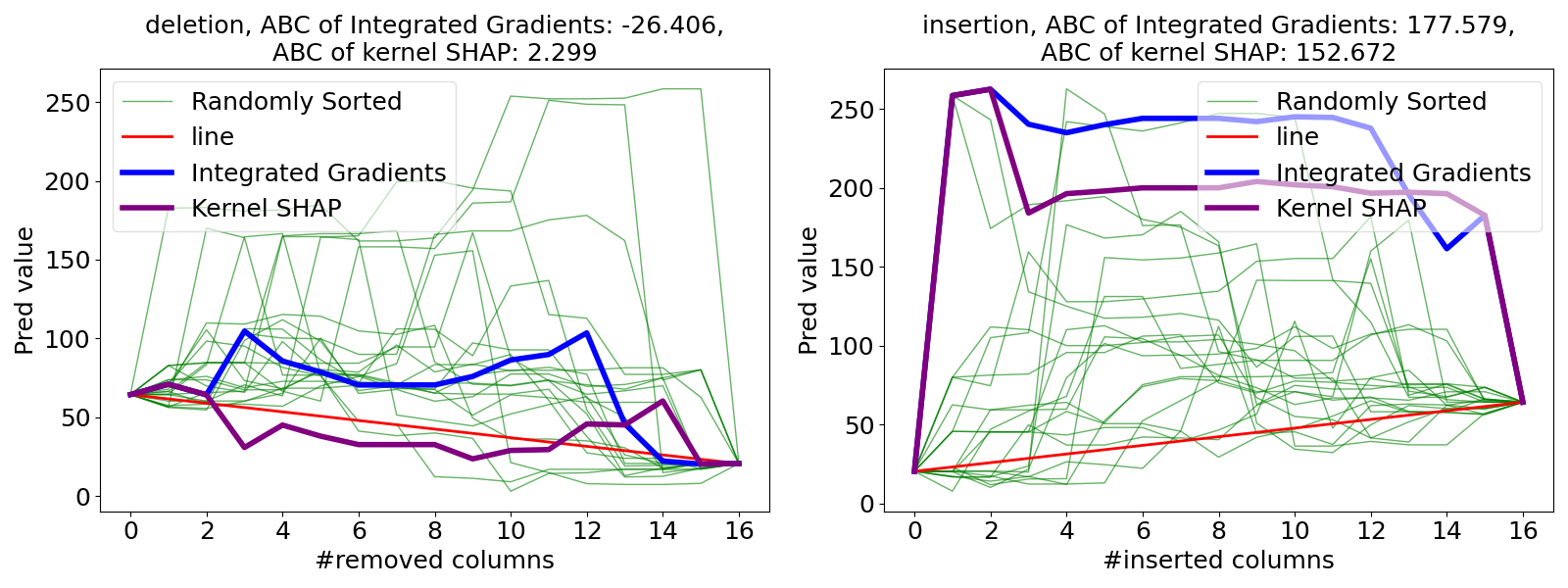}
    \caption{Insertion and deletion curves
        for the most asymmetrical data in Figure~\ref{fig:corr_insertiondeletion_oto} that has
        over 150 GeV in insertion ABC and near to zero GeV
        in deletion ABC.
    }
    \label{fig:oto_asym}
\end{figure}

\begin{figure}[t]
    \includegraphics[width=\linewidth]{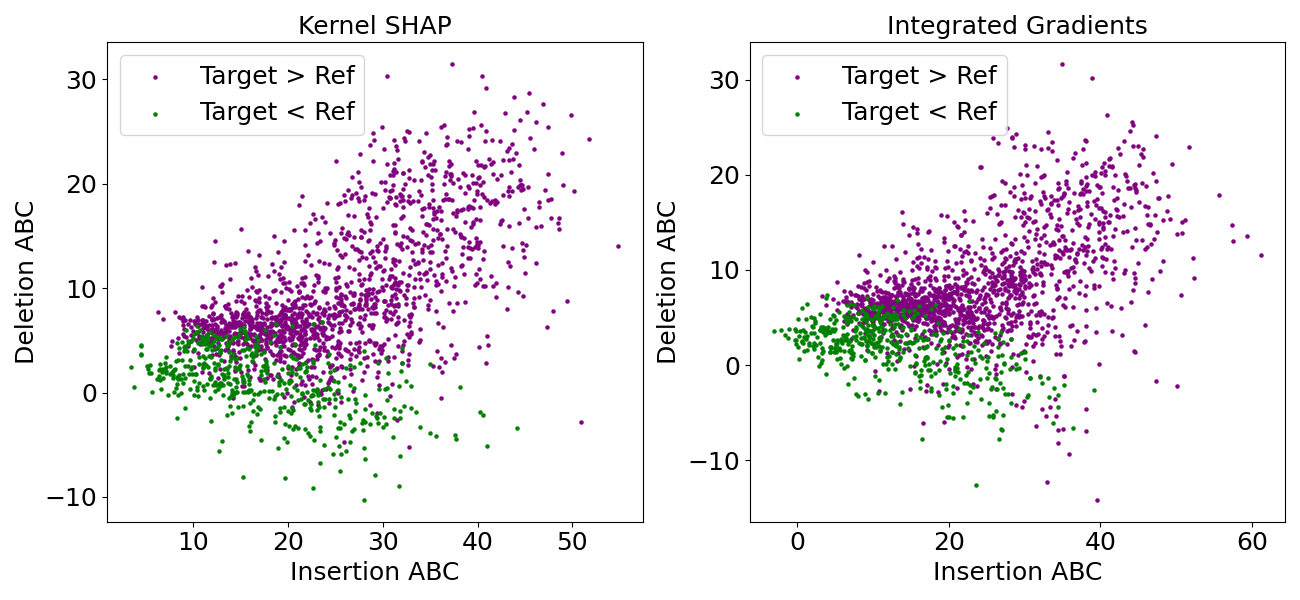}
    \caption{For the CERN data with the average policy, these
        figures plot deletion versus insertion ABCs.  The left plot is for KS and the right is for IG.}
    \label{fig:corr_insertiondeletion_ave}
\end{figure}

\begin{table}
    \centering
    \begin{tabular}{lccc}
        \toprule
                       & Number of Different Cols & \multicolumn{2}{c}{Correlation Coefficients}               \\
        Policy         & (Total Number $= 16$)    & $\phe$KS                                     & $\phe$IG    \\
        \midrule
        Counterfactual & 16$\phantom{.991}$       & $\phm$0.820                                  & $\phm$0.799 \\
        One-to-One     & 14.991                   & $-$0.276                                     & $-$0.297    \\
        Average        & 16$\phantom{.991}$       & $\phm$0.587                                  & $\phm$0.499 \\
        \bottomrule
    \end{tabular}
    \caption{For three policies on $(\bsx,\bsx')$ in
        the CERN data:
        the average number of $j$ with $\bsx_j\ne\bsx'_j$ and
        the correlation between insertion and deletion ABCs,
        for both KS and IG.}
    \label{tb:assym_diff}
\end{table}

The ABC including other XAI methods are aggregated in
Table~\ref{tb:diffABC_oto}.
The relationships among their magnitudes are almost same
as those in the counterfactual policy of Table~\ref{tb:diffABC_cern}.
It is also confirmed that insertion tests have  larger ABC values
than deletion tests with same setup in general and their
differences are larger than those of the counterfactual policy.
This point supports our previous discussion about their asymmetry.
The computation of Input$\times$Gradient does not take the reference data into account, and
this might explain why Input$\times$Gradient has comparatively scores compared to other methods in the average policy.

The average number of different columns between reference data and target data and
correlation coefficients of ABCs in two kinds of tests are summarized in Table \ref{tb:assym_diff}.
All sixteen columns have different values in the counterfactual policy by definition.
The deviation from sixteen in the value in one-to-one policy is mostly from
the two charges of the data, which can take only two levels, $+1$ or $-1$.
In this sense, the reference data in average policy is unphysical data
since it has charges that are near zero.
The correlation coefficients between two tests also vary between policies.
We note also that
the behavior of the two ABCs in the average policy strongly depends on whether $f(\bsx) > f(\bsx')$ or $f(\bsx) < f(\bsx')$
as seen in the scatter plots Figure \ref{fig:corr_insertiondeletion_ave}.

\section{RemOve and Retrain Methods}\label{sec:roarexample}
In this section we compare KS and IG via ROAR (RemOve And Retrain)  \citep{hooker2018benchmark}.
{ROAR is significantly more expensive
to study than the other methods
we consider as it requires retraining
the models, and so we did not apply it
to all of the methods.
We opted to apply it just
to Kernel SHAP and IG.  We chose IG as our
representative fast method because IG can be used
on more general models than DeepLIFT can.
We chose Kernel SHAP as the
other method because it had best or near
best ABC values on our numerical examples.
The task in ROAR is about which variables
are important to the model's accuracy and not about
which variables are important to any specific
prediction.  As a result the values in ROAR
are not comparable to the other AUC and
ABC values that we have computed.
}

The original proposal of ROAR measures the drop in accuracy for
image classification tasks and applying it to regression tasks with tabular data
raises the same issues
as extending insertion/deletion tests to regression.
We measure the test loss on
    {held out data} as a measure of
retrained model performance.
Retraining procedures are taken with an increasing number of
removed features at each quantile in $\{0.1, 0.3, 0.5, 0.7, 0.9, 1.0\}$, where $1.0$ means that all features are removed.
The model architecture and hyperparameters
of retrained models are
the original models as described in Sections~\ref{sec:housingprices} and~\ref{sec:cerncolldata}.

\begin{figure}[t]
    \begin{center}
        \includegraphics[width=\linewidth]{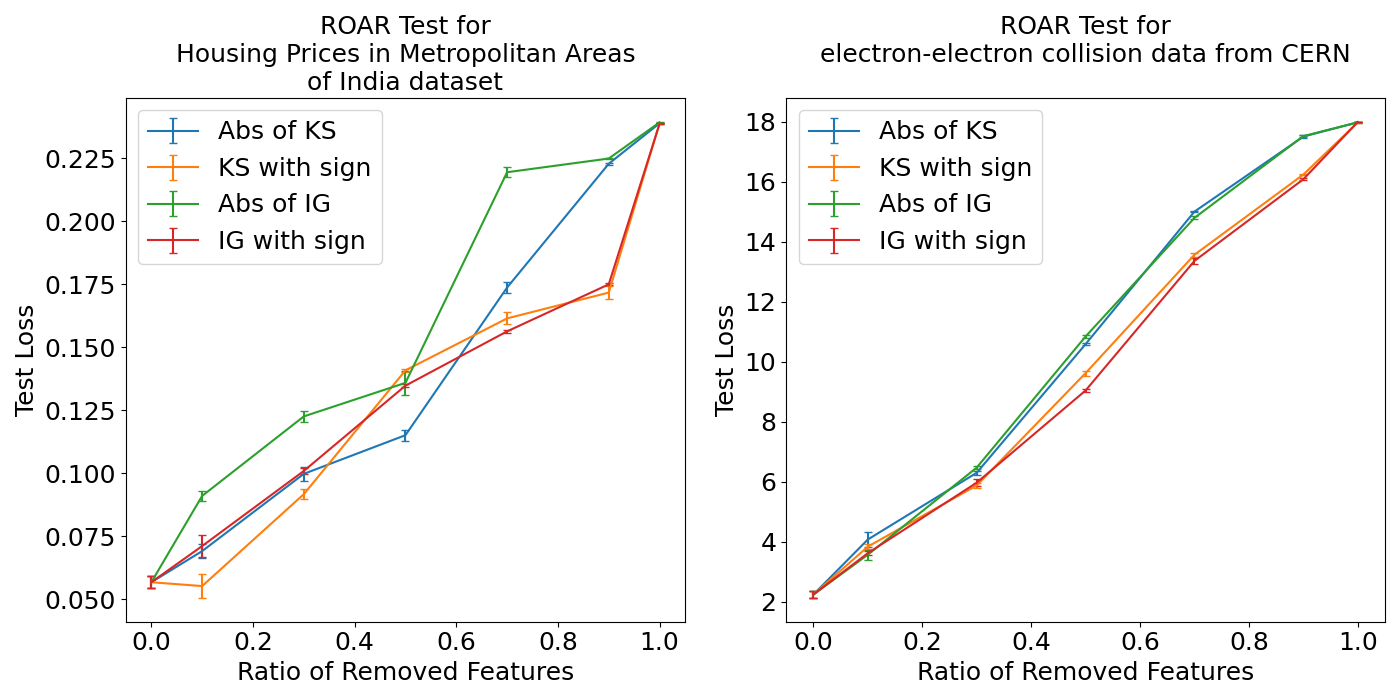}
        \caption{The result of ROAR test for the Bangalore housing dataset and the CERN Electron Collision Data.}
        \label{fig:roar_results}
    \end{center}
\end{figure}

To use ROAR we must
decide how to remove features in the data.
Removing in the original ROAR
algorithm means padding pixels of the
original images
with noninformative values such as gray
levels.
In our experiments,
the important features are overwritten by average values for continuous features and by modes for discrete features.
These average and modal values are taken from the training data.

The results of our ROAR calculations
are shown in Figure \ref{fig:roar_results}.
Features to be removed are sorted both by their
absolute values and by their original signed
values for both KS and IG.
    {The error bars
        show plus or minus one standard error
        computed from five replicates.}

Since ROAR in this experiment measures the
Huber
loss on the test points,
it should be unaffected by the signs of the attributions and
sensitive to their magnitude. For this reason,
sorting features by the absolute values of their
attributions should give a better score than
sorting by their signed values.
This is a contrasting point to insertion and deletion tests.

    {
        The results in Figure~\ref{fig:roar_results} are surprising to us.  The curves we see are very nearly straight lines connecting the loss with all features present to the loss with no features present, so there is very little area between
        them and a straight line connecting the end points.
        This means that the loss in accuracy from variables deemed
        most important is about the same as those deemed
        less important.
        This could be because neither KS nor IG are able to identify important predictors for this task. It could also be that the majority of predictor variables in this data can be
        replaced by a combination of some other predictors which then prevents a large
        reduction in accuracy from removing a subset
        of predictors prior to retraining.

        A second surprise is that the signed ordering, that we used
        as a control
        that should have been beaten by the absolute ordering
        was nearly as good as the absolute ordering.

    }

As before KS and IG are comparable, though here
both seem disappointing.
Using IG with variables sorted by their absolute values even
came out superior to KS in
the Bangalore housing  dataset.

\section{Conclusion}
\label{sec:con}
In this paper we have extended insertion and deletion
tests to regression problems.  That includes getting formulas
for the effects of interactions on the AUC and ABC measures,
finding the expected area between the insertion/deletion
curve under random variable ordering, and replacing the
horizontal axis by a more appropriate straight line
reference.  We gave a condition under which
sorting variables by their Shapley value will optimize
ABC as well as constructing an example where that
does not happen.

We compared six methods and several policies on two
datasets.  We find that overall the Kernel SHAP gave
the best areas.  The much faster Integrated Gradients
method was nearly as good.
In order to even run IG in settings with binary
variables, some strategy for using continuum
values must be employed.
We opted for the simplest choice
of just casting the booleans to real values.

A very natural policy question is whether to prefer
insertion or deletion.  \cite{petsiuk2018rise} consider
both and do not show a strong preference for one over
the other.  They use deletion when comparing a real image
to a blank image (deleting real pixels by replacing them with zeros). They use insertion when comparing a real image to a blurred one (inserting real pixels into the blurred image).
In other words the choice between insertion and deletion
is driven by the counterfactual point.
In the regression setting both inputs points could be
real data.  By studying both insertion and deletion we
have seen that they can differ. A natural way to break
the tie is to sum the ABC values for both insertion
and deletion.  Under a completely random permutation
the expected value of that sum is zero.  See Appendix~\ref{sec:abc4del}. In our examples, IG closely matches
KS for both insertion and deletion, so it also matches
their sum.

\section*{Acknowledgement}
This work was supported by the U.S.\ National Science
Foundation grants IIS-1837931 and DMS-2152780,
and by Hitachi, Ltd.  We thank Benjamin Seiler for
helpful comments. Comments
from three anonymous reviewers have helped
us improve this paper.

\bibliographystyle{apalike}
\bibliography{deletion_insertion}
\appendix

\section{Detailed Model Descriptions of the Experiments}\label{sec:modeldetails}
This appendix provides some background details on the experiments conducted in this article.
\subsection{The Example in Image Classification}\label{sec:imageexample}
Here we summarize how the example of
insertion and deletion tests in image classification
shown in Figure \ref{fig:cub} was computed.
The image is from \citet{wah2011caltech} and
the model is pretrained for classification in ImageNet \citep{russakovsky2015imagenet}
whose architecture is EfficientNet-B0 \citep{tan2019efficientnet}.
The preprocessing of the model includes cropping
the center of the image to make it square
as shown in the saliency map of Figure \ref{fig:cub}.

The saliency map is computed using SmoothGrad of
\cite{smilkov2017smoothgrad}.
It averages IG computations over 300
randomly generated baseline images of Gaussian noise.
The model output are distributed to the latent features
in the first convolutional layer
(implemented in Captum \citep{kokhlikyan2020captum} as
layer integrated gradient).
That layer has 32 blocks each of which
has a $112\times112$ grid of $3$ pixels times
$3$ pixels. The effect of any pixel is summed
over those 32 blocks and over all $3\times3$
patterns that contain it.
In the insertion and deletion test,
this saliency map is then the same size, $224\times224$,
as the preprocessed original image.
The reference image for both insertion and deletion tests
was a black image.


The AUCs for several different choices of reference data are summarized in Table \ref{tb:cub_ABC}. Note that the definition of AUCs are different from the main material of this paper and small deletion AUC is better in this situation.
The reference image has a very significant effect
on the AUCs.

\begin{table}
    \centering
    \begin{tabular}{lcc}
        \toprule
        Reference Image & Deletion AUC & Insertion AUC \\
        \midrule
        Blurred         & 0.981        & 0.740         \\
        Mean            & 0.663        & 0.187         \\
        White           & 0.366        & 0.175         \\
        Black           & 0.270        & 0.201         \\
        \bottomrule
    \end{tabular}
    \caption{AUCs for the albatross
        example of Figure \ref{fig:cub} using various
        reference images. The parameter for blurring the image is the same one \citet{petsiuk2018rise} used.}
    \label{tb:cub_ABC}
\end{table}

\subsection{Model for the Bangalore Housing Data}\label{sec:housingprices}
The detail of the model used in Section~\ref{sec:expindia} is
summarized in Table \ref{tb:modelindia}.
Those hyperparameters, including the number of intermediate layers
were obtained from a
hyperparameter search using Optuna \citep{akiba2019optuna}.
The model is optimized with Huber loss.
Each intermediate layer is accompanied with parametric ReLU \citep{he2015delving} and dropout layers. The dropout ratio is common in all of them.

\begin{table}
    \centering
    \begin{tabular}{ll}
        \toprule
        Hyperparameter    & Value                   \\
        \midrule
        Dropout Ratio     & 0.10031                 \\
        Learning Rate     & $1.7389 \times 10^{-2}$ \\
        Number of Neurons & [333--465--86--234]     \\
        Huber Parameter   & $1.0$                   \\
        \bottomrule
    \end{tabular}
    \caption{Parameters of the MLP model for
        the Bangalore housing data.}
    \label{tb:modelindia}
\end{table}

\subsection{Other XAI methods in Tables \ref{tb:diffABC}, \ref{tb:diffABC_cern} and  \ref{tb:diffABC_oto}}
\label{sec:otherxais}
The implementation details of the other XAI methods than KS and IG which appear in Tables \ref{tb:diffABC}, \ref{tb:diffABC_cern} and \ref{tb:diffABC_oto} are summarized in this subsection.

We use the implementations on Captum \citep{kokhlikyan2020captum}
for those methods, DeepLIFT \citep{shrikumar2017learning}, Vanilla Grad \citep{simonyan2013deep}, Input$\times$Gradient \citep{shrikumar2016not} and LIME \citep{ribeiro2016should} with default set arguments.
Inputs of Input$\times$Gradient are those after applying
Z-score normalizing in the electron-electron collision data from CERN.
Zeros in binary vectors are replaced by a small negative value
($-10^{-4}$) in Input$\times$Gradient to avoid degeneration in the
Metropolitan Areas of India dataset.
As we use parametric ReLU as the activation functions in our model,
it is also treated as a usual nonlinear function in the DeepLIFT calculations.
The reference values that can be set in LIME and DeepLIFT are
chosen as same data to KS and IG,
depending on their policy.

\subsection{CERN Electron Collision Data}\label{sec:cerncolldata}
The hyperparameters for the model used in Section~\ref{sec:cernexp} are
given in Table \ref{tb:modelcern}.
They were obtained from a
hyperparameter search using Optuna \citep{akiba2019optuna}.
Each  intermediate layer is  a parametric ReLU with dropout. The dropout ratio is common to all of the layers.
The performance for test data is depicted in Figure \ref{fig:test_acc2}.  The model is overall very accurate but the very highest values are systematically underestimated.

\begin{figure}[t]
    \begin{center}
        \includegraphics[width=0.5\linewidth]{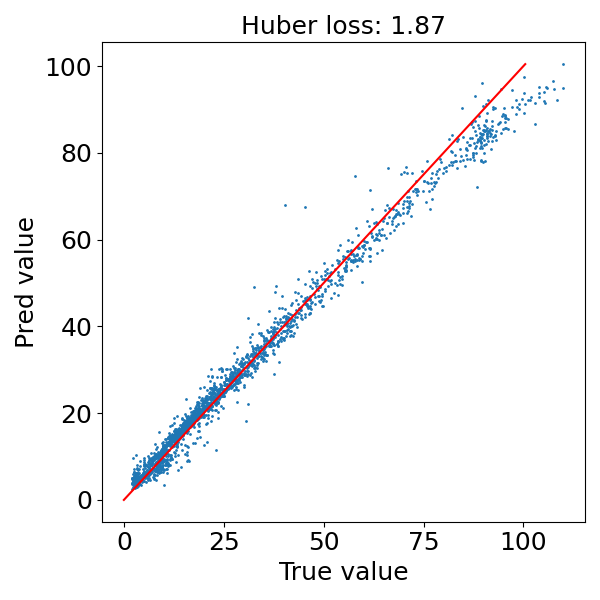}
        \caption{This figure plots estimated versus true
            invariant masses for held out points of the electron-electron collision
            data from CERN.
        }
        \label{fig:test_acc2}
    \end{center}
\end{figure}

\begin{table}
    \centering
    \begin{tabular}{ll}
        \toprule
        Hyperparameter    & Value                         \\
        \midrule
        Dropout Ratio     & 0.11604                       \\
        Learning Rate     & $1.9163 \times 10^{-4}$       \\
        Number of Neurons & [509--421--65--368--122--477] \\
        Huber Parameter   & $1.0$                         \\
        \bottomrule
    \end{tabular}
    \caption{Parameters of the MLP model for the CERN Electron Collision Data.}
    \label{tb:modelcern}
\end{table}

\section{Proof of Theorem~\ref{thm:aucfromdeltaf}\label{sec:auctheory}}

Here we prove that $\auc = \sum_{u\subseteq1{:}n}(n-\lceil u\rceil +1)\Delta_uf$.  We use the anchored decomposition that we define next. We also connect that decomposition to some areas of the literature.

The anchored decomposition is a kind of
high dimensional model representation (HDMR)
that represents a function of $n$ variables
by a sum of $2^n$ functions, one per subset of $1{:}n$
where the function for $u\subseteq1{:}n$ depends on
$\bsx$ only through $\bsx_u$.
The best known HDMR is the ANOVA of
\cite{fish:mack:1923}, \cite{hoef:1948}, \cite{sobo:1969} and \cite{efro:stei:1981}
but there are others.  See
\cite{kuo:sloa:wasi:wozn:2010}.

The anchored decomposition goes back
at least to \cite{sobo:1969}.  It does
not require a distribution on the inputs.  Instead of centering
higher order interaction terms by subtracting expectations, which
don't exist without a distribution, it centers by subtracting
values at default or anchoring input points.  We only need
it for functions  on $\{0,1\}^n$ and without loss of generality
we take the anchor to be all zeros.

We  use $\bszero,\bsone\in\{0,1\}^n$ for
vectors of all ones and all zeros, respectively.
For $u\subseteq1{:}n$ we write
$e_u=\bszero_u{:}\bsone_{-u}$ generalizing the
standard basis vectors $e_j$.
The function we need to study is
$g:\{0,1\}^n\to\real$ where
$$g(e_u)=f(\bsx_u{:}\bsx'_{-u}),$$
gives the values in the curve we study.

The anchored decomposition of $g:\{0,1\}^n\to \real$ is
\begin{align*}
    g(\bsz)             & = \sum_{u\subseteq1{:}n}g_{u}(\bsz),\quad\text{with} \\
    g_{\emptyset}(\bsz) & = g(\bszero),\quad\text{and for $|u|>0$},            \\
    g_{u}(\bsz)         & = g(\bsz_u{:}\bszero_{-u})
    -\sum_{v\subsetneq u}g_{v}(\bsz).
\end{align*}

The main effect in an anchored decomposition is
$g_{j}(\bsz) = g(\bsz_j{:}\bszero_{-j})-g(\bszero)$
and the two factor term for indices $j\ne k$ is
\begin{align*}
    g_{\{j,k\}}(\bsz)
     & = g(\bsz_{\{j,k\}}{:}\bszero_{-\{j,k\}})
    -g_{j}(\bsz) -g_{k}(\bsz)-g_\emptyset(\bsz) \\
     & = g(\bsz_{\{j,k\}}{:}\bszero_{-\{j,k\}})
    -g(\bsz_j{:}c_{-j})
    -g(\bsz_k{:}c_{-k})
    +g(\bszero).
\end{align*}
There is an inclusion-exclusion-M\"obius formula
\begin{align*}
    g_{u}(\bsz) & = \sum_{v\subseteq u}(-1)^{|u-v|}g_v(\bsz_v{:}\bszero_{-v}).
\end{align*}
See for instance \cite{kuo:sloa:wasi:wozn:2010}.
The anchored decomposition
is also called cut-HDMR \citep{alis:rabi:2001} in chemistry,
and finite differences-HDMR in global sensitivity analysis \citep{sobo:2003}.
When $f$ is the value function in a Shapley value
context, the values $g_u(\bsone)$ are known as Harsanyi dividends
\citep{hars:1959}.
Many of the quantities
we use here also feature prominently in the
study of Boolean functions $f:\{0,1\}^n\to\{0,1\}$
\citep{odon:2014}.

The next Lemma is from \cite{mase:owen:seil:2019}. We include the short
proof for completeness.
\begin{lemma}\label{lem:binanchdecomp}
    For integer $n\ge1$,
    let $f:\{0,1\}^n\to\real$ have the anchored decomposition
    $g(\bsz) = \sum_{u\subseteq1{:}n}g_{u}(\bsz)$.
    Then for $w\subseteq1{:}n$,
    \begin{align*}
        g_{u}(\bse_w) = g_{u}(\bsone)1_{u\subseteq w},
    \end{align*}
    where $\bse_w=\bsone_w{:}\bszero_{-w}$.
\end{lemma}
\begin{proof}
    The inclusion-exclusion formula for the binary anchored
    decomposition is
    $$
        g_{u}(\bsz) = \sum_{v\subseteq u}(-1)^{|u-v|}g(\bsz_v{:}\bszero_{-v}).
    $$
    Suppose that $z_j=0$ for $j\in u$. Then, splitting up the alternating sum
    \begin{align*}
        g_{u}(\bsz) & = \sum_{v\subseteq u-j}(-1)^{|u-v|}
        (g(\bsz_v{:}\bszero_{-v}) - f(\bsz_{v+j}{:}\bszero_{-v-j}))
        = 0
    \end{align*}
    because $\bsz_v{:}\bszero_{-v}$
    and $\bsz_{v+j}{:}\bszero_{-v-j}$ are the same point when $z_j=0$.
    It follows that $g_{u}(\bse_w)=0$ if $u\not\subseteq w$.

    Now suppose that $u\subseteq w$.
    First $g_{u}(\bsz)=g_{u}(\bsz_u{:}\bsone_{-u})$
    because $g_{u}$ only depends on $\bsz$ through $\bsz_u$.
    From $u\subseteq w$ we have $(\bse_w)_u=\bsone_u$.
    Then
    $g_{u}(\bse_w)=g_u(\bsone_u{:}\bsone_{-u})=
        g_u(\bsone)$, completing the proof.
\end{proof}

We are now ready to state and
prove our theorem expressing the AUC in
terms of the anchored decomposition.
Without loss of generality it takes
$\pi$ to be the identity permutation.

\begin{theorem}\label{thm:aucfromdeltaf}
    Let $f:\bsx\to\real$ and let $g:\{0,1\}^n\to\real$
    be defined by $g(\bse_u)=f(\bsx'_u{:}\bsx_{-u})$,
    and let $\pi = (1,2,\dots,n)$.
    Then the $\auc$ given by~\eqref{eq:defauc} satisfies
    \begin{align}\label{eq:aucfromdeltaf}
        \auc = \sum_{u\subseteq1{:}n}g_u(\bsone)(n-\lceil u\rceil+1)
        = \sum_{u\subseteq1{:}n}(n-\lceil u\rceil+1)\Delta_uf.
    \end{align}
\end{theorem}
\begin{proof}
    First, $\tilde\bsx^{(j)} = \bsx_{1{:}j}{:}\bsx'_{-\{1{:}j\}}$.
    Then
    $$
        \auc = \sum_{j=0}^nf(\tilde\bsx^{(j)})
        = \sum_{j=0}^ng(e_{1:j})
        = \sum_{j=0}^n\sum_{u\subseteq1{:}n}g_u(e_{1:j}).
    $$
    Next using Lemma~\ref{lem:binanchdecomp},
    we find that $\auc$ equals
    \begin{align*}
        \auc & =\sum_{j=0}^n\sum_{u\subseteq1{:}n}g_u(\bsone)1_{u\subseteq 1:j} \\
             & =
        \sum_{j=0}^n\sum_{u\subseteq1{:}n}g_u(\bsone)1_{\lceil u\rceil \le j}   \\
             & =
        \sum_{u\subseteq1{:}n}g_u(\bsone)(n-\lceil u\rceil+1).
    \end{align*}
    Finally
    \begin{align*}
        g_u(\bsone)
        = \sum_{v\subseteq u}(-1)^{|u-v|}g(e_v)
        = \sum_{v\subseteq u}(-1)^{|u-v|}
        f(\bsx'_v{:}\bsx_{-v})=\Delta_uf.\qquad\qedhere
    \end{align*}
\end{proof}

\subsection{ABC for Deletion}\label{sec:abc4del}
Now suppose that we use the deletion strategy of replacing $x_j$ by $x_j'$ in the opposite order from that used above, meaning that we change variables thought to most decrease $f$ first.
Then letting $\lfloor u \rfloor$ be the index of the smallest element of $u\subseteq1{:}n$, with $\lfloor\emptyset\rfloor=n+1$ by convention, we get by the argument in Theorem~\ref{thm:aucfromdeltaf},
$$
    \auc' = \sum_{u\subseteq1{:}n}\lfloor u\rfloor\Delta_uf.
$$
Our area between the curves, ABC, measure for deletion is
\begin{align*}
    \abc' & =\frac{n+1}2\bigl( f(\bsx)+f(\bsx')\bigr)-\auc'                          \\
          & = \sum_{u\ne\emptyset}\Bigl(\frac{n+1}2-\lfloor u\rfloor\Bigr)\Delta_uf.
\end{align*}
If we sum the two ABC measures we get
$$
    \abc+\abc' = \sum_{u\ne\emptyset}(n-\lceil u\rceil-\lfloor u\rfloor+1)\Delta_uf.
$$

Proposition~\ref{prop:meanceiling} gave us a formula
for $\e(\lceil \pi(u)\rceil)$ where $\pi(u)$ is the image
of the set $u$ under a uniform random permutation of $1{:}n$.
By symmetry, we know that
$$
    \Pr( \lfloor \pi(u)\rfloor =\ell)=\Pr(\lceil\pi(u)\rceil = n-\ell+1)
$$
for $0\le \ell\le n+1$.
As a result $\e( \lceil u\rceil + \lfloor u\rfloor)=n+1$ from which
$$\e( \abc+\abc')=0.$$

\subsection{Monotonicity}\label{sec:monotonicity}

Here we prove a sufficient condition under which ranking variables
in decreasing order by their Shapley value gives the order that maximizes the insertion AUC.
We suppose that
$f(\bsz) = h(a(\bsz))$ where $a(\bsx)$ is an additive function
on $\{0,1\}^n$ and $h:\real\to\real$ is strictly increasing.
An additive function on $\bsz\in\{0,1\}^n$ takes the form
$$a(\bsz)=\gamma_0 + \sum_{j=1}^n\gamma_jz_j.$$
By choosing $h(w)=\sigma(w)\equiv (1+\exp(-w))^{-1}$ we can study logistic
regression probabilities, while $h(w)=w$ accounts for
those same probabilities on the logit scale.  By choosing $h(w)=\exp(w)$
we can include naive Bayes. 
Taking $h(w)$ to be the leaky ReLU function we can compare the
importance of the inputs to a neuron at some position within a
network.

Logistic regression is ordinarily expressed
as $\Pr(Y=1\giv\bsx)=\sigma(\beta_0+\bsx^\tran\beta)$.
Then $\Pr(Y=1\giv\bsx')=\sigma(\beta_0+\bsx'^\tran\beta$).
    If we select $\bsz\in\{0,1\}^n$ with $z_j=1$
    indicating that we choose $x'_j$ for the $j$'th component
    and $j=0$ indicating that we choose $x_j$ for the $j$'th
    component, then
    $$f(\bsz) = \sigma\Bigl(\beta_0+\sum_{j=1}^dz_j(x_j'-x_j)\beta_j\Bigr).
    $$
    In other words we take $\gamma_j=(x_j'-x_j)\beta_j$
    and $\gamma_0=\beta_0$ to define the function on $\{0,1\}^n$
    that we study.

    The composite function is
$f(\bsz) = h(a(\bsz))$.
    We suppose without loss of generality that
$\beta_1\ge\beta_2\ge\cdots\ge\beta_n$.
    Then the Shapley values satisfy
    $$\phi_1\ge\phi_2\ge\cdots\ge\phi_n.$$
    To see this, suppose that $u\subseteq1{:}n$ with
$\ell,\ell'\not\in u$ and $\ell<\ell'$.
    Then
    $$f(e_{u\cup\ell})-f(e_{u\cup \ell'})
        = h\Bigl(\beta_0+\sum_{j\in u}\beta_j + \beta_\ell\Bigr)
        -h\Bigl(\beta_0+\sum_{j\in u}\beta_j + \beta_{\ell'}\Bigr)
        \ge 0.
    $$
    It follows that the incremental gains from adding $\ell$ to
    any set $u$ not containing $\ell$ and $\ell'$ is never smaller
    than that from adding $\ell'$ and hence $\phi_\ell\ge\phi_{\ell'}$.

    Now suppose that we arrange the variables in some
    order $\pi(j)$ where $\pi(\cdot)$ is a permutation of $1{:}n$.
    We then get an AUC of
    $$
        \auc(\pi)=\sum_{j=0}^n h\Bigl( \beta_0 + \sum_{\ell=1}^j\beta_{\pi(\ell)}\Bigr),
    $$
    where the summation over $\ell$ is zero for $j=0$.
    Now let $\pi'$ be a different permutation that swaps positions
$r$ and $r+1$ in $\pi$ where $1\le r<n$.
    It has
    \begin{align*}
        \auc(\pi') & =\sum_{j=0}^n h\Bigl( \beta_0 + \sum_{\ell=1}^j\beta_{\pi'(\ell)}\Bigr) \\
                   & =\sum_{j=0}^n
        h\Bigl( \beta_0 + \sum_{\ell=1}^{j}\bigl(
            \ind_{\ell<r}\beta_{\pi(\ell)}
            +\ind_{\ell=r}\beta_{\pi(r+1)}
            +\ind_{\ell=r+1}\beta_{\pi(r)}
            +\ind_{\ell>r+1}\beta_{\pi(\ell)}
            \bigr)
        \Bigr).
    \end{align*}
    Therefore $\auc(\pi')$ and $\auc(\pi)$
    only differ in the summand for $j=r$ and so
    \begin{align*}
        \auc(\pi')-\auc(\pi) & =
        h\Bigl( \beta_0 + \sum_{\ell=1}^{r-1}\beta_{\pi(\ell)}+\beta_{\pi(r+1)}\Bigr)
        -h\Bigl( \beta_0 + \sum_{\ell=1}^{r+1}\beta_{\pi(\ell)}\Bigr).
    \end{align*}
    Now if $\beta_{\pi(r+1)}<\beta_{\pi(r)}$ we get
$\auc(\pi')<\auc(\pi)$. As a result, any
    maximizer $\pi$ of $\auc$ must have
$\beta_{\pi(r+1)}\ge\beta_{\pi(r)}$ for all $r=1,\dots,n-1$.

    Next we consider integrated gradients for this setting,
    assuming that $h$ is differentiable with $h'>0$.
    The gradient is then
$h'(\bsz)\beta$. The gradient at any point then sorts the
inputs in the same order as the Shapley value.  Therefore
any positive linear combination of those gradient evaluations
sorts the inputs into this order which then optimizes
the deletion AUC.

\end{document}